%% file: GoodAction_ICML.tex
\documentclass[english]{article}

\input{preamble.tex}

\usepackage{natbib}
\usepackage{microtype}
\usepackage{graphicx}
\usepackage{subcaption}
\usepackage{booktabs} 
\usepackage{multicol}
\usepackage{natbib}
\usepackage{enumitem}
\usepackage{relsize}
\usepackage{hyperref}
\usepackage{adjustbox}

\usepackage[accepted]{icml2021}

\icmltitlerunning{Lenient Regret and Good-Action Identification in Gaussian Process Bandits}

\begin{document}

\twocolumn[
\icmltitle{Lenient Regret and Good-Action Identification in Gaussian Process Bandits}




\begin{icmlauthorlist}
    \icmlauthor{Xu Cai}{1}
    \icmlauthor{Selwyn Gomes}{1}
    \icmlauthor{Jonathan Scarlett}{1,2}
\end{icmlauthorlist}

\icmlaffiliation{1}{Department of Computer Science, National University of Singapore}
\icmlaffiliation{2}{Department of Mathematics \& Institute of Data Science, National University of Singapore}

\icmlcorrespondingauthor{Xu Cai}{caix@u.nus.edu}
\icmlcorrespondingauthor{Selwyn Gomes}{selwyn@comp.nus.edu.sg}
\icmlcorrespondingauthor{Jonathan Scarlett}{scarlett@comp.nus.edu.sg}

\icmlkeywords{Gaussian Process; Bandits}

\vskip 0.3in
]



\printAffiliationsAndNotice{}  

\newcommand{\lgamma}{\mathlarger{\gamma}}
\newcommand{\sizeiconf}{0.475\linewidth}
\newcommand{\sizeiarxiv}{0.3\linewidth}

\input{abstract.tex}
\input{main.tex}

\section*{Acknowledgement}
This work was supported by the Singapore National Research Foundation (NRF) under grant number R-252-000-A74-281.

\bibliographystyle{icml2021}
\bibliography{JS_References}

\newpage

\onecolumn

{\Huge \bf \centering Supplementary Material \par}

\medskip
{\Large \bf \centering Lenient Regret and Good-Action Identification \\ [1.5mm] in Gaussian Process Bandits (ICML 2021) \par}

\medskip
{\large \bf \centering Xu Cai, Selwyn Gomes, and Jonathan Scarlett \par}
\medskip

\appendix

\input{appendixA.tex}

\input{appendixBCD.tex}

\end{document}

%% file: preamble.tex
\usepackage[T1]{fontenc}
\usepackage[utf8]{inputenc}
\usepackage{amsthm}
\usepackage{amsmath}
\usepackage{amssymb}
\usepackage{graphicx}
\usepackage{dsfont}
\usepackage{cite}
\usepackage{amsmath}
\usepackage{array}
\usepackage{multirow}
\usepackage{caption}
\usepackage{color}

\usepackage{multicol}

\theoremstyle{plain}

\theoremstyle{plain}

\theoremstyle{plain}
\newtheorem{lem}{\protect\lemmaname}
\theoremstyle{plain}
\newtheorem{thm}{\protect\theoremname}
\theoremstyle{plain}
  
\theoremstyle{definition}

\theoremstyle{definition}

\theoremstyle{definition}

\makeatother
  
\usepackage{babel} 

\providecommand{\claimname}{Claim}
\providecommand{\lemmaname}{Lemma}
\providecommand{\propositionname}{Proposition}
\providecommand{\theoremname}{Theorem}
\providecommand{\corollaryname}{Corollary} 
\providecommand{\definitionname}{Definition}
\providecommand{\assumptionname}{Assumption}
\providecommand{\remarkname}{Remark}

\newcommand{\mytildee}[1]{\mkern 1.5mu\widetilde{\mkern-1.5mu#1\mkern-0mu}\mkern 0mu}
  
\DeclareMathOperator*{\argmax}{arg\,max}

\newcommand{\openone}{\mathds{1}}

\newcommand{\lcb}{\mathrm{lcb}}
\newcommand{\ucb}{\mathrm{ucb}}
\newcommand{\lcbbar}{\overline{\mathrm{lcb}}}
\newcommand{\ucbbar}{\overline{\mathrm{ucb}}}
\newcommand{\Tcbad}{\mathcal{T}_{\mathrm{bad}}}
\newcommand{\Ttilde}{\widetilde{T}}
\newcommand{\Deltatil}{\widetilde{\Delta}}
\newcommand{\Ostar}{O^{*}}

\newcommand{\sigmatil}{\widetilde{\sigma}}

\newcommand{\Nc}{\mathcal{N}}

\newcommand{\Rtilde}{\mytildee{R}}

\newcommand{\xvtilde}{\widetilde{\mathbf{x}}}
\newcommand{\xv}{\mathbf{x}}

\newcommand{\yv}{\mathbf{y}}

\newcommand{\Fc}{\mathcal{F}}

\newcommand{\Tc}{\mathcal{T}}

\newcommand{\EE}{\mathbb{E}}
\newcommand{\PP}{\mathbb{P}}

\newcommand{\ZZ}{\mathbb{Z}}

\newcommand{\Iv}{\mathbf{I}}

\newcommand{\Kv}{\mathbf{K}}
\newcommand{\kv}{\mathbf{k}}

\newcommand{\bzero}{\boldsymbol{0}}

%% file: abstract.tex
\begin{abstract}
    In this paper, we study the problem of Gaussian process (GP) bandits under relaxed optimization criteria stating that any function value above a certain threshold is ``good enough''.  On the theoretical side, we study various {\em lenient regret} notions in which all near-optimal actions incur zero penalty, and provide upper bounds on the lenient regret for GP-UCB and an elimination algorithm, circumventing the usual $O(\sqrt{T})$ term (with time horizon $T$) resulting from zooming extremely close towards the function maximum.  In addition, we complement these upper bounds with algorithm-independent lower bounds.  On the practical side, we consider the problem of finding a single ``good action'' according to a known pre-specified threshold, and introduce several good-action identification algorithms that exploit knowledge of the threshold.  We experimentally find that such algorithms can often find a good action faster than standard optimization-based approaches.
\end{abstract}

%% file: main.tex
\vspace*{-1.5ex}
\section{Introduction} \label{sec:intro}

Gaussian Process (GP) methods have recently gained popularity as a highly effective tool in finding the optimum $f(\xv^*)$ of a black-box function $f$ \cite{Sha16}, with a particularly notable advantage being sample efficiency.  Alongside the practical developments, the theory of GP bandits has also seen several interesting advances.  The results can broadly be classified according to whether the mathematical model adopted is Bayesian (i.e., the function is assumed to be random and drawn from a GP) or non-Bayesian (i.e., the function is deterministic and assumed to have a bounded norm in a suitably-defined Reproducing Kernel Hilbert Space (RKHS)), and the same GP-based algorithms can often be applied in a unified manner in these two settings.

Perhaps the most prominent class of existing results concerns cumulative regret bounds that scale with the time horizon as $\sqrt{T}$ or higher, and simple regret bounds that show convergence to the optimum at a rate of $\frac{1}{\sqrt T}$ or slower \cite{Sri09,Con13,Bog16a,Cho17,Jan20}.  While algorithm-independent lower bounds show such behavior to be unavoidable \cite{Sca17a,Sca18a}, their proofs suggest these regret terms are predominantly dictated by the hardness of zooming increasingly close to the locally-quadratic maximum (Bayesian setting), or of finding a very small and narrow bump hidden in an otherwise flat function (RKHS setting).  In practice, one may not be concerned with the distinction between being ``very close'' vs.~``extremely close'' to the maximum, or one may not mind missing the existence of a very small bump.  In this sense, there is potentially a wide gap between standard theoretical guarantees and practical desiderata.

Motivated by these considerations, we investigate theory and algorithms for Gaussian process bandits under various notions that only seek to find ``good enough'' actions, where an action $\xv$ is considered good if $f(\xv)$ is within a certain threshold $\Delta>0$ of the optimum $f(\xv^*)$.  In particular, following a recent work in the multi-armed bandit literature \cite{Mer20} and focusing on the non-Bayesian RKHS setting, we study {\em lenient regret} notions that incur no penalty for good-actions.  We show that this circumvents the $\sqrt{T}$ term appearing (and being unavoidable) in the standard cumulative regret, and that GP-UCB \cite{Sri09} and an elimination algorithm \cite{Con13} can instead incur a significantly smaller lenient regret such as ${\rm poly}(\log T)$, or even just a constant value (i.e, $O(1)$) depending on $\Delta$.

In addition, we consider the related problem of finding a single point whose function value exceeds some pre-specified threshold $\eta>0$ (we may set $\eta=f(\xv^*)-\Delta$), which we call the {\em good-action identification} problem.  This problem may be of interest, for example, in the context of hyperparameter tuning, where narrowing down a near-optimal configuration may be prohibitively expensive, so one may instead resort to seeking a ``sufficiently good'' configuration.  We connect the good-action identification problem to the notion of lenient regret, and provide novel algorithms that are specifically targeted to this setting and exploit the knowledge of $\eta$.    We empirically observe that these algorithms can improve on standard optimization-based approaches, using both synthetic and non-synthetic functions.


\subsection{Related Work} \label{sec:related}

Theoretical works on GP bandits have focused mainly on the cumulative regret (see \eqref{eq:R_T} below), and in some cases the simple regret (see \eqref{eq:simple} below).  Perhaps most related to our work are the analyses of GP-UCB in \cite{Sri09,Cho17}, and of elimination-based algorithms in \cite{Con13,Bog16a}, as well as the algorithm-independent lower bounds in \cite{Sca17a,Cai20}.  

The preceding works provide near-tight scaling laws for the squared-exponential (SE) kernel, while incurring larger gaps for the Mat\'ern kernel; however, these gaps have been narrowed in a recent line of works \cite{Val13,Jan20,She20}.  Other theoretical studies include those for the noiseless setting \cite{Bul11,Gru10} and the Bayesian setting \cite{Sca18a,She17}, but these are less relevant to the present paper.

Our work is motivated by recent works in the multi-armed bandit (MAB) literature studying various notions of lenient regret \cite{Mer20} and good-arm identification \cite{Kan19,Kat20}.  Like with these works, we seek to show that such notions can be attained with significantly fewer samples; however, the associated algorithms, results, and analyses have minimal similarity with these works, due to the very different continuous action space along with smoothness assumptions.

Some works on GP bandits have sought to incorporate prior information such as monotonicity \cite{Li17} and knowledge of the function maximum \cite{Ngu20a}, but to our knowledge, none have considered notions relating to lenient regret and good-action identification.

Finally, the problem of identifying an action whose function value exceeds a given threshold is related to {\em level-set estimation} (LSE), which has been studied using GP methods \cite{Bry05,Got13,Bog16a,She19}.  However, the goal of LSE is to classify the {\em entire domain} into points falling above/below the threshold, whereas our focus is on finding just a single point above the threshold.  Thus, applying LSE methods to our setting would amount to unnecessarily solving a harder problem as an intermediate step.

\section{Problem Setup} \label{sec:setup}

We consider the problem of sequentially optimizing an unknown function $f$ on a compact domain $D$, taking $D = [0,1]^d$ for concreteness.  In each round indexed by $t=1,\dotsc,T$, the algorithm selects $\xv_t \in D$ and observes a noisy sample $y_t=f(\xv_t) + z_t$, with $z_t\sim \Nc(0,\sigma^2)$. 

We focus on the non-Bayesian RKHS setting (briefly turning to the Bayesian setting in Section \ref{sec:bayesian}), adopting the assumption that $f \in \Fc_k(B)$, where $\Fc_k(B)$ denotes the set of all functions whose RKHS norm $\|f\|_k$ is upper bounded by some constant $B > 0$.  We consider arbitrary choices of the kernel $k(\xv,\xv')$ for the most part, but will sometimes pay particular attention to the squared exponential (SE) and Mat\'ern kernels \cite{Ras06}, parametrized by the length-scale $l$ (both cases) and the smoothness parameter $\nu$  (Mat\'ern only).  Throughout the paper, we assume normalization such that $k(\xv,\xv) \le 1$ for all $\xv \in D$, with equality for the SE and Mat\'ern kernels.

Despite considering the non-Bayesian RKHS setting, it is useful to consider a `fictitious' Bayesian GP posterior: Given a sequence of inputs $(\xv_1, \dots, \xv_t)$ and their noisy observations $(y_1, \dots, y_t)$, the posterior distribution under a $\mathrm{GP}(\boldsymbol{0}, k)$ prior and $\Nc(0,\lambda)$ sampling noise\footnote{Since this is a fictitious update model, the parameter $\lambda$ may differ from the true noise variance $\sigma^2$.} is also Gaussian, with mean and variance given by
\begin{align}
\mu_{t}(\xv) &= \kv_t(\xv)^T\big(\Kv_t + \lambda \mathbf{I}_t \big)^{-1} \yv_t,  \label{eq:posterior_mean} \\ 
\sigma_{t}^2(\xv) &= k(\xv,\xv) - \kv_t(\xv)^T \big(\Kv_t + \lambda \mathbf{I}_t \big)^{-1} \kv_t(\xv), \label{eq:posterior_variance}
\end{align}
where $\kv_t(\xv) = \big[k(\xv_i,\xv)\big]_{i=1}^t$, and $\Kv_t = \big[k(\xv_t,\xv_{t'})\big]_{t,t'}$ is the kernel matrix. 

The most widely-adopted performance measure in the literature is the (standard) cumulative regret, defined as 
\begin{equation}
    R_T = \sum_{t=1}^T \big( f(\xv^*) - f(\xv_t) \big), \label{eq:R_T}
\end{equation}
where $\xv^*$ denotes any maximizer of $f$.  Another popular notion is the {\em simple regret} $r^{(T)}$, in which the algorithm returns an additional point $\xv^{(T)}$ (not necessarily a sampled one) after $T$ rounds, and
\begin{equation}
    r^{(T)} =  f(\xv^*) - f(\xv^{(T)}). \label{eq:simple}
\end{equation}

\subsection{Lenient Regret}\label{sec:lenient}

In light of the motivation in the introduction, and following recent study of \cite{Mer20} for the multi-armed bandit setting, we consider notions of {\em lenient regret} in which no penalty is incurred when $f(\xv_t)$ is within $\Delta$ of the optimum, for some small $\Delta > 0$.  In view of this property, we henceforth refer to $\xv \in D$ satisfying $f(\xv) \ge f(\xv^*) - \Delta$ as {\em good actions}, and to other $\xv$ as {\em bad actions}.

In generic notation, we consider (cumulative) lenient regret notions of the form
\begin{equation}
    \Rtilde_{T} = \sum_{t=1}^T \Phi( r_t ), \quad r_t = f(\xv^*) - f(\xv_t)
\end{equation}
for some function $\Phi(\cdot)$ such that $\Phi(r) = 0$ for all $r \le \Delta$ (whereas $\Phi(r) = r$ would recover \eqref{eq:R_T}).

We focus our attention on the following three choices of $\Phi$ suggested in \cite{Mer20}:
\begin{itemize}[leftmargin=5ex,itemsep=0ex,topsep=0.25ex]
    \item {\bf Indicator:} $\Phi^{\rm ind}(r) = \openone\{r > \Delta\}$, implying that $\Rtilde^{\rm ind}_{T}$ counts the number of bad actions.
    \item {\bf Large Gap:} $\Phi^{\rm gap}(r) = r\cdot \openone\{ r > \Delta \}$, implying that $\Rtilde^{\rm gap}_{T}$ only accumulates the simple regret of bad actions.
    \item {\bf Hinge:} $\Phi^{\rm hinge}(r) = \max (r - \Delta, 0)$, implying that $\Rtilde^{\rm hinge}_{T}$ accumulates the distances of bad actions' function values to the good-action threshold.
\end{itemize}
These functions are illustrated in Figure \ref{fig:regret_compare}.  Intuitively, one might expect the large-gap regret and hinge regret to behave similarly when $\Delta$ is small, whereas the indicator regret may be larger due to the rapid transition from zero to one; our theory will support this intuition.

\begin{figure}
    \begin{centering}
        \includegraphics[width=0.325\textwidth]{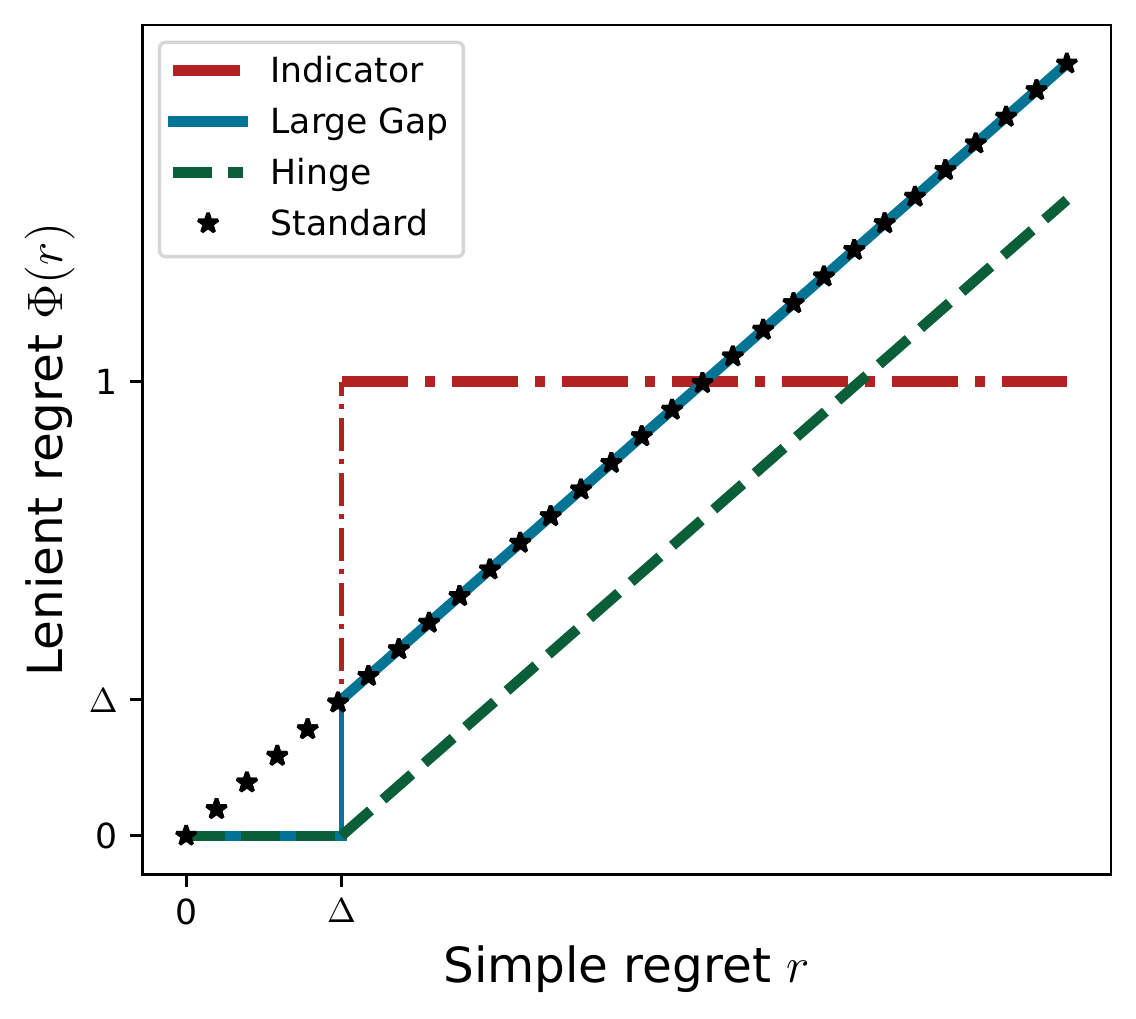}
        \par
    \end{centering}
    \caption{Illustration of three choices of $\Phi$ for the lenient regret, along with the choice that yields the standard regret. \label{fig:regret_compare}} 
\end{figure}

\subsection{Good-Action Identification}

In addition to the above lenient regret notions that increase in a cumulative manner, it is also of interest to consider the case that the algorithm is only required to return a single point, and is considered successful if that point is a good action (i.e., its function value is within $\Delta$ of the optimum).  If the time horizon $T$ is fixed and the returned point is $\xv^{(T)}$, then this is equivalent to attaining simple regret at most $\Delta$ (see \eqref{eq:simple}).  Since several theoretical guarantees are already known for the simple regret (e.g., see \cite{Bog16a,She17,Sca17a}), we do not explore them further in this paper, though analogous guarantees can indeed be inferred via simple modifications to our lenient regret analysis.

Instead, in order to move further beyond what is already known, we consider the problem of {\em fixed-threshold good-arm identification}, where an action $\xv \in D$ is considered {\em good} if $f(\xv^*) \ge \eta$ for some pre-specified threshold $\eta > 0$, and {\em bad} otherwise.  This coincides with our above notion of ``good'' and ``bad'' actions when $\eta = f(\xv^*) - \Delta$.

On the other hand, in contrast to our studies of lenient regret, when $\eta$ is pre-specified, it is natural to assume that it is {\em known to the algorithm}.  Thus, in Section \ref{sec:alg}, we introduce algorithms for good-action identification that exploit the prior knowledge of $\eta$, and provide experimental evidence that this can be beneficial in Section \ref{sec:experiments}.

\subsection{GP-UCB Algorithm}

In our study of the lenient regret, we focus on the widely-considered Gaussian process upper confidence bound (GP-UCB) algorithm \cite{Sri09}, which selects the $t$-th point $\xv_t$ to maximize the acquisition function
\begin{equation}
    \alpha^{\rm UCB}_t(\xv) = \mu_{t-1}(\xv) + \beta^{1/2}_t \sigma_{t-1}(\xv),
\end{equation}
for some suitably-chosen exploration parameter $\beta_t$.  We use the following well-known result \cite{AbbasiThesis} (see also \cite{Cho17}) to select $\beta_t$.  Here and subsequently, we make use of the {\em maximum information gain}, which is widely used in the GP bandit literature, and is defined as
\begin{equation}
\label{eq:max_info_gain}
\lgamma_t = \max_{\xv_1, \dots, \xv_t} \frac{1}{2} \ln  \det(\Iv_t + \lambda^{-1}\Kv_t)
\end{equation}
with $\Kv_t$ defined following \eqref{eq:posterior_variance}. 

\begin{lem}
    \label{lem:conf}
    {\em \cite{AbbasiThesis}}
    For any $\lambda > 0$ and $f \in \Fc_k(D)$ with $\| f \|_k \leq B$, under the choice\footnote{We follow the convention of \cite{Sri09} and equate this expression with $\beta_t^{1/2}$, whereas some other works denote the right-hand side by $\beta_t$.}
    \begin{equation} 
        \beta_t^{1/2} = B + \sigma \lambda^{-1/2} \sqrt{2(\lgamma_{t-1} + \ln(1 / \delta))}, \label{eq:beta}
    \end{equation}
    we have with probability at least $1-\delta$ that $\lcb_t(\xv) \le f(\xv) \le \ucb_t(\xv)$ for all $t$ and $\xv \in D$, where
    \begin{align}
        \ucb_t(\xv) &= \mu_{t-1}(\xv) + \beta_t^{1/2} \sigma_{t-1}(\xv), \label{eq:ucb} \\
        \lcb_t(\xv) &= \mu_{t-1}(\xv) - \beta_t^{1/2} \sigma_{t-1}(\xv). \label{eq:lcb}
    \end{align}
    and where $\mu_{t-1}(\cdot)$ and $\sigma_{t-1}(\cdot)$ are given in \eqref{eq:posterior_mean}--\eqref{eq:posterior_variance}.
\end{lem}

\subsection{Elimination Algorithm} \label{sec:elim}

In addition to GP-UCB, we consider a simple algorithm that selects actions with the maximum uncertainty, while using the confidence bounds to eliminate suboptimal actions.  While we are not aware of this exact algorithm being used before, it is of a very standard form, and can be viewed as a simplified variant of elimination algorithms such as GP-UCB-PE \cite{Con13} and truncated variance reduction \cite{Bog16a}.

The idea is to define a set of {\em potential maximizers}
\begin{equation}
    M_t = \Big\{ \xv \in M_{t-1} \,:\, \ucb_t(\xv) \ge \max_{\xv'} \lcb_t(\xv') \Big\} \label{eq:Mt}
\end{equation}
and observe that when the UCB and LCB functions in \eqref{eq:ucb}--\eqref{eq:lcb} provide valid confidence bounds, $M_t$ contains $\xv^*$ while also eliminating suboptimal points.

With the above definitions in place, the algorithm initializes $M_0 = D$ and $t = 1$, and repeats the following:
    \begin{itemize}[leftmargin=5ex,itemsep=0ex,topsep=0.25ex]
        \item[(i)] Select $\xv_t = \argmax_{\xv \in M_{t-1}} \sigma_{t-1}(\xv)$; 
        \item[(ii)] Observe $y_t$ and update the posterior (i.e., $\mu_{t}(\cdot)$ and $\sigma_t(\cdot)$) and set of potential maximizers (i.e., $M_t$ in \eqref{eq:Mt}), and increment $t$.
    \end{itemize}

\section{Lenient Regret Bounds}  \label{sec:results}

In this section, we provide our main theoretical results on the lenient regret of GP-UCB and the elimination algorithm.  The proofs are deferred to Appendix \ref{sec:proofs}.

\subsection{Lenient Regret of GP-UCB}

Our first main result is as follows.

\begin{thm} \label{thm:main1}
    {\em (Lenient Regret of GP-UCB)} Define
    \begin{equation}
        N_{\max} = \max \Big\{ N \,:\, N\leq \frac{C_1\lgamma_N\beta_T}{\Delta^2} \Big\}, \label{eq:Nmax}
    \end{equation}
    where $C_1  = \frac{8\lambda^{-1}}{\log(1+\lambda^{-1})}$.
    For any $f \in \Fc_k(B)$ and any $\delta \in (0,1)$, $\lambda > 0$, and $\Delta > 0$, GP-UCB run with the choice of $\beta_t$ in \eqref{eq:beta} satisfies the following lenient regret bounds with probability at least $1-\delta$: 
    \begin{itemize}[leftmargin=5ex,itemsep=0ex,topsep=0.25ex]
        \item[(i)] $\Rtilde_T^{\rm ind} \le N_{\max}$;
        \item[(ii)] $\Rtilde_T^{\rm hinge} \le \Rtilde_T^{\rm gap} \le  \frac{ C_1 \lgamma_{N_{\max}} \beta_T }{\Delta}$.
    \end{itemize}
\end{thm}

{\bf Specialization to SE and Mat\'ern kernels.} To bound $N_{\max}$ under the widely-considered SE and Mat\'ern kernels, we use the following known bounds on $\lgamma_t$:
\begin{itemize}[leftmargin=5ex,itemsep=0ex,topsep=0.25ex]
    \item For the SE kernel, we have $\lgamma_t = \Ostar( (\log t)^d )$ \cite{Sri09};
    \item For the Mat\'ern-$\nu$ kernel, we have $\lgamma_t = \Ostar( t^{\frac{d}{2\nu + d}})$ \cite{Vak20a}.
\end{itemize}
Here and subsequently, $\Ostar(\cdot)$ hides dimension-independent logarithmic factors, and will also hide $\log \log T$ factors in expressions for which $\log T$ factors are present.  In addition, we treat $B$, $\sigma$, $\lambda$, $d$, $l$, and $\nu$ as being constant as $T$ increases.

We have from \eqref{eq:beta} that $\beta_t = \Theta( \lgamma_t )$, and hence, the condition defining $N_{\max}$ in \eqref{eq:Nmax} weakens to $\frac{N}{(\log N)^d} \le \Ostar\big( \frac{(\log T)^d}{\Delta^2} \big)$.  $N_{\max}$ is upper bounded by the $N$ for which this expression holds with equality; from this fact, we can deduce that $\log N = \Theta\big( \log \frac{(\log T)^d}{\Delta^2} \big) = O\big( \log\frac{1}{\Delta} + \log\log T \big)$, and hence $N_{\max} \le \Ostar\big( \frac{(\log T \cdot \log\frac{1}{\Delta})^d}{\Delta^2} \big)$.

For the Mat\'ern-$\nu$ kernel, assuming $d < 2\nu$,  the condition defining $N_{\max}$ in \eqref{eq:Nmax} weakens to $N^{\frac{2\nu}{2\nu+d}} \le \Ostar\big( \frac{T^{\frac{d}{2\nu+d}}}{\Delta^2} \big)$, and we obtain $N_{\max} \le \Ostar\big( \frac{T^{\frac{d}{2\nu}}}{\Delta^{2+\frac{d}{\nu}}} \big)$.

These bounds on $N_{\max}$ directly bound $\Rtilde_T^{\rm ind}$, and can also be substituted into Theorem \ref{thm:main1} to deduce similar (albeit more complicated) bounds on $\Rtilde_{T}^{\rm hinge}$ and $\Rtilde_{T}^{\rm gap}$; in particular, the dominant term is $\frac{(\log T)^d}{\Delta^2}$ for the SE kernel.  

{\bf Comparison to standard regret bounds.} The lenient regret bounds can be considerably smaller than the $O(\sqrt{T \lgamma_T \beta_T})$ standard cumulative regret bounds for GP-UCB \cite{Cho17}.  For instance, for the SE kernel, the reduction is from $\sqrt{T}{\rm poly}(\log T)$ to simply ${\rm poly(\log T)}$.  More generally, we notice that the standard regret bound is only sublinear when $\lgamma_T \beta_T = o(T)$, and limiting our attention to this regime along with $\Delta = \Theta(1)$, we immediately deduce from \eqref{eq:Nmax} that $N_{\max} = o(T)$, which in turn implies that the bound $\Rtilde_T^{\rm gap} \le  \frac{ C_1 \lgamma_{N_{\max}} \beta_T }{\Delta}$ is at most $O(\lgamma_T \beta_T)$ (and possibly much smaller), which is itself much smaller than $\sqrt{T \lgamma_T \beta_T}$ (since $\sqrt{\lgamma_T \beta_T} = o(\sqrt T)$).

{\bf Discussion.} While Theorem \ref{thm:main1} indicates that the lenient regret of GP-UCB can be much smaller than the standard regret, it still grows unbounded as $T \to \infty$, due to the presence of $\beta_T$.  It is conceivable that an algorithm could have {\em bounded} lenient regret with high probability, if it manages to find a region of points within $\Delta$ of the optimum and subsequently only samples in that region.  However, GP-UCB will not satisfy such a property when $\lim_{t \to \infty}\beta_t = \infty$ (as is the case for all known variants with theoretical guarantees), since the growing exploration constant ensures that even suboptimal regions are returned to after long enough.\footnote{In Appendix \ref{sec:intersect}, we discuss the possibility of using GP-UCB with confidence bounds intersected across time.}

\subsection{Lenient Regret of the Elimination Algorithm} \label{sec:lenient_elim}

In light of the limitations of Theorem \ref{thm:main1} discussed above, we present the following improved lenient regret bounds for the elimination algorithm.

\begin{thm} \label{thm:main2}
    {\em (Lenient Regret of the Elimination Algorithm)} Define
    \begin{equation}
        N'_{\max} = \max \Big\{ N \,:\, N\leq \frac{4C_1\lgamma_N\beta_N}{\Delta^2} \Big\}. \label{eq:Nmax'}
    \end{equation}
    For any $f \in \Fc_k(B)$ and any $\delta \in (0,1)$, $\lambda > 0$, and $\Delta > 0$, the elimination algorithm in Section \ref{sec:elim} run with the UCB and LCB functions in Lemma \ref{lem:conf} satisfies the following lenient regret bounds with probability at least $1-\delta$:
    \begin{itemize}[leftmargin=5ex,itemsep=0ex,topsep=0.25ex]
        \item[(i)] $\Rtilde_T^{\rm ind} \le N'_{\max}$;
        \item[(ii)] $\Rtilde_T^{\rm hinge} \le \Rtilde_T^{\rm gap} \le 2B + \frac{8C_1 \lgamma_{N'_{\max}}\beta_{N'_{\max}}}{\Delta}$;
    \end{itemize}
    where $C_1  = \frac{8\lambda^{-1}}{\log(1+\lambda^{-1})}$.
\end{thm}

The main difference compared to Theorem \ref{thm:main1} is that $\beta_T$ in \eqref{eq:Nmax} is replaced by $4\beta_N$.  The latter is highly preferable, since we have $N_{\max} = o(T)$ in the scaling regimes of interest, as discussed following Theorem \ref{thm:main1}.  In particular, the regret bounds are now independent of $T$, with the intuition being that all bad actions are eventually eliminated.  

However, this improvement has an important practical caveat, namely, the algorithm may degrade much less gracefully than GP-UCB when the kernel is unknown or learned online.  This is because kernel mismatch in the earlier rounds may lead to $\xv^*$ being eliminated, and in principle even the entire domain could get eliminated.  In view of this trade-off, better understanding the interaction between kernel uncertainty and lenient regret remains an interesting direction for future work.

{\bf Specialization to the SE and Mat\'ern kernels.}  Following a similar argument to the one following Theorem \ref{thm:main1}, we find that $N_{\max} \le \Ostar\big( \frac{(\log\frac{1}{\Delta})^{2d}}{\Delta^2} \big)$.  For the Mat\'ern kernel, we require $\frac{d}{2\nu + d} < \frac{1}{2}$ (or equivalently, $d < 2\nu$) for $N'_{\max}$ to be finite; note that analogous constraints are also required for the optimization regret bounds in \cite{Sri09,Cho17} to be non-trivial.  When $d < 2\nu$, some simple manipulations give $N'_{\max} \le \Ostar\big( \frac{1}{\Delta^{2(1+\frac{d}{2\nu - d})}} \big)$, in particular becoming closer to $\frac{1}{\Delta^2}$ as $\nu$ increases. 

\subsection{Algorithm-Independent Lower Bounds}

Lower bounds on the standard regret for noisy GP bandit optimization were introduced in \cite{Sca17a}, and were refined in \cite{Cai20} via a distinct but related analysis.  The idea is to consider functions with a small ``bump'' that is hard for the algorithm to locate, with the height of the bump being tuned to attain the best possible cumulative regret lower bound.  It turns out that the analysis techniques of \cite{Cai20} readily transfer to the setting of lenient regret, but with a larger bump height (namely, $O(\Delta)$) in order to prevent the scenario of trivially having zero lenient regret regardless of the points chosen.  This yields the following.

\begin{thm} \label{thm:main_lb}
    {\em (Lower Bounds on the Lenient Regret)} Fix $\delta \in \big(0,\frac{1}{3}\big)$, $\Delta \in \big(0,\frac{1}{2}\big)$, $B > 0$, and $T \in \ZZ$, and suppose that $\frac{\Delta}{B} = O(1)$ with a sufficiently small implied constant,\footnote{Note that if $\Delta > 2B$ then the lenient regret is trivially zero, since any $f \in \Fc_k(B)$ must have $\max_x |f(x)| \le B$.} and that the dimension $d$ and kernel parameters are constant.  Then, for any algorithm, the lenient regret must be lower bounded as follows:
    \begin{itemize}[leftmargin=2ex,itemsep=0ex,topsep=0.25ex]
    \item For the SE kernel, there exists $f \in \Fc_k(B)$ such that the following holds with probability at least $\delta$:\footnote{We state our lower bounds as failure events that hold with probability at least $\delta$, which is equivalent to saying that all algorithms are unable to attain a success probability of $1-\delta$.}
        \begin{itemize}[leftmargin=4ex,itemsep=0ex,topsep=0.25ex]
            \item[(i)] $\Rtilde_T^{\rm ind} \ge \Omega\big( \min\big\{ T, \frac{\sigma^2}{\Delta^2} \big(\log\frac{B}{\Delta}\big)^{d/2} \log \frac{1}{\delta} \big\} \big)$;
            \item[(ii)] $\Rtilde_T^{\rm hinge} \ge \Omega\big( \min\big\{ T\Delta, \frac{\sigma^2}{\Delta} \big(\log\frac{B}{\Delta}\big)^{d/2} \log \frac{1}{\delta} \big\} \big)$  (and $\Rtilde_T^{\rm gap} \ge \Rtilde_T^{\rm hinge}$).
        \end{itemize}
    \item For the Mat\'ern kernel, there exists $f \in \Fc_k(B)$ such that the following holds with probability at least $\delta$:
        \begin{itemize}[leftmargin=4ex,itemsep=0ex,topsep=0.25ex]
            \item[(i)] $\Rtilde_T^{\rm ind} \ge \Omega\big( \min\big\{ T, \frac{\sigma^2}{\Delta^2} \big(\frac{B}{\Delta}\big)^{d/\nu} \log \frac{1}{\delta} \big\} \big)$;
            \item[(ii)] $\Rtilde_T^{\rm hinge} \ge \Omega\big( \min\big\{ T\Delta, \frac{\sigma^2}{\Delta} \big(\frac{B}{\Delta}\big)^{d/\nu} \log \frac{1}{\delta} \big\} \big)$ (and $\Rtilde_T^{\rm gap} \ge \Rtilde_T^{\rm hinge}$).
        \end{itemize}
    \end{itemize}
\end{thm}

To compare with the upper bounds in Theorem \ref{thm:main2}, we again treat $B$, $\sigma^2$, and $\delta$ as constants, focusing on the dependence on $\Delta$.  In addition, we focus on the scaling regimes of primary interest in which each $\min\{\cdot,\cdot\}$ is achieved by the second term (in the other case, there are $\Theta(T)$ bad arm pulls, which is analogous to the standard regret being linear in $T$).

For the SE kernel, the upper and lower bounds match up to the replacement of $d/2$ by $2d$ in the exponent, and thus, we have proved that $\frac{1}{\Delta^2}$ (for $\Rtilde_T^{\rm ind}$) or $\frac{1}{\Delta}$ (for $\Rtilde_T^{\rm gap}$ and $\Rtilde_T^{\rm hinge}$) is indeed the correct leading term. 

For the Mat\'ern kernel, wider gaps remain between the upper and lower bounds, as is also the case for the standard cumulative regret of GP-UCB \cite{Cho17} and arm elimination \cite{Con13} compared to the lower bounds \cite{Sca17a}.  These gaps for the standard cumulative regret can be closed using the impractical SupKernelUCB algorithm \cite{Val13}, or partially closed using covering techniques that remain effective in practice \cite{Jan20}.  However, these algorithms are also more difficult to analyze, and would likely need further modifications to remove the dependence on $T$ in the same way as Theorem \ref{thm:main2}.  Hence, the analysis of their lenient regret is left for possible future work.

\subsection{Upper Bounds for the Bayesian Setting} \label{sec:bayesian}

Throughout the paper, we have focused on the non-Bayesian setting in which $f \in \Fc_k(B)$.  However, since our upper bounds are centered around the validity of the confidence bounds in Lemma \ref{lem:conf}, they also naturally extend to the Bayesian setting in which $f \sim \mathrm{GP}(\boldsymbol{0}, k)$ and the exploration constants $\beta_t$ are suitably modified.  This is most straightforward in the finite-domain setting, in which we can set $\beta_t = 2\log\frac{|D| t^2 \pi^2}{ 6\delta }$ \cite{Sri09}, along with $\lambda = \sigma^2$ in \eqref{eq:posterior_mean}--\eqref{eq:posterior_variance}.  

In the continuous-domain Bayesian setting, the changes are slightly less straightforward, but we can again follow \cite{Sri09} under the assumption of the sample paths being Lipschitz-continuous with high probability.  The analysis (but not the algorithm) then makes use of a discretization argument that slightly increases the uncertainty of any given point in the analysis.  This added uncertainty amounts to replacing $\Delta$ by $\Delta - \epsilon$ for arbitrarily small $\epsilon > 0$ in the bounds, having a negligible impact for any fixed $\Delta > 0$.  If $\Delta$ is considered to be decreasing as $T$ increases, then the analysis can additionally be modified so that $\epsilon$ decreases.  The details are omitted for the sake of brevity.

\section{Good-Action Identification Algorithms} \label{sec:alg}

Our theory suggests that the GP-UCB algorithm, which was introduced for studying the standard regret notion \cite{Sri09}, is also effective in finding ``good enough'' actions, either according to the lenient regret with parameter $\Delta$ or the fixed-threshold setting with parameter $\eta$.  In this section, we complement our theory by introducing additional practical algorithms that are specifically geared towards the fixed-threshold setting, and explicitly incorporate knowledge of the threshold $\eta$ with the goal of finding a point satisfying $f(x) \ge \eta$.  Experimental evaluations will be performed in Section \ref{sec:experiments}.

\subsection{Probability of Being Good (PG)}

The early work of \cite{Kus64} suggested to choose the next query point as the one which has the highest {\em probability of improvement} (PI) over the current maximum $f(\xv^{+})$, where $\xv^{+} = \argmax_{\xv \in \{\xv_1,\dotsc,\xv_{t-1}\}}f(\xv)$.  Motivated by this idea, we consider choosing the action as the one having the highest {\em probability of being good} (PG):
\begin{align}
    \alpha^{\rm PG}_t(\xv) &= \PP_{t-1}[ f(\xv)\geq \eta ] = \Phi \Big(\frac{\mu_{t-1}(\xv)-\eta}{\sigma_{t-1}(\xv)}\Big), \label{eq:pg_acq}
\end{align}
where $\PP_{t-1}[\cdot]$ denotes the posterior probability after $t-1$ queries (and subsequently similarly for $\EE_{t-1}[\cdot]$), and $\Phi(\cdot)$ denotes the cumulative density function (CDF) of the standard Gaussian distribution.  

Since $\Phi(\cdot)$ is an increasing function, we can equivalently maximize the argument $\frac{\mu_{t-1}(\xv)-\eta}{\sigma_{t-1}(\xv)}$ in \eqref{eq:pg_acq}; this is more numerically stable due to avoiding very small $\Phi(\cdot)$ values.

\subsection{Expected Improvement Over Good (EG)}

By choosing the next query point as the one having the highest {\em expected improvement} (EI) over the current maximum $f(\xv^{+})$, one can account for the  {\em amount} of improvement into consideration, rather than just the probability of improvement \cite{Moc78}.   While any good action is considered sufficient in our setting, it is still natural to analogously consider the {\em expected improvement over good} (EG) selection rule:
\begin{align}
    \alpha^{\rm EG}_t(\xv) &= \EE_{t-1}[\max\{0, f(\xv) - \eta\}] \nonumber \\
    &= (\mu_{t-1}(\xv) - \eta)\Phi(u_{\xv}) + \sigma_{t-1}(\xv)\phi(u_{\xv}), \label{eq:eg_acq}
\end{align}
where $u_{\xv}=\frac{\mu_{t-1}(\xv)-\eta}{\sigma_{t-1}(\xv)}$, and $\phi$ denotes the {probability density function} (PDF) of the standard Gaussian distribution.

\subsection{Good-Action Search (GS)} \label{sec:gs}

Motivated by the success of entropy search and its variants \cite{Hen12,Her14,Wan17}, we can consider being ``less myopic'' and looking forward one step based on the current posterior.  Specifically, if we consider choosing $\xv$ as the next point, then the resulting $y$ will be random, and appending $(\xv,y)$ to the data set will form a new posterior $(\mu_t,\sigma_t)$.  We can then consider seeking to maximize $\EE_{y}\big[ \PP_t[y_0^* \ge \eta] \big]$, where $y_0^* = \max_{\xv} f(\xv)$,\footnote{The subscript of $0$ is used to emphasize representing a function value {\em before} adding noise.} $y$ is distributed according to the current posterior, and $\PP_t[\cdot]$  implicitly depends on $(\xv,y)$ and represents the updated posterior.

Since the exact computation of $\EE_{y}\big[ \PP_t[y^* \ge \eta] \big]$ is difficult, we can instead use a surrogate based on randomly-drawn samples as follows:
\begin{align}
    \alpha^{\rm GS}_t(\xv) = \frac{1}{K}\sum_{y_0^*\in Y_0^*}\openone\{y_0^* \geq \eta\}, \label{eq:gs_acq_1}
\end{align}
where $Y_0^*$ a set of $K$ samples of maximum function values upon choosing $\xv$, which can be generated in an identical manner to max-value entropy search (MES) via a Gumbel distribution approximation \cite{Wan17}.  

It may be the case that all of the $\xv$ lead to a set $Y_0^*$ in which all of the points are below $\eta$; when this occurs, we choose the $\xv$ that produced the highest value of $\max_{y_0^* \in Y_0^*} y_0^*$.

\subsection{Other Algorithms}

In Appendix \ref{sec:other}, we additionally present two good-action identification algorithms that build on (i) Thompson sampling and (ii) action elimination.  However, as discussed therein, these algorithms appear to rely more heavily on prior knowledge that is typically unavailable, and so we omit them from our experiments in the following section.

\section{Experiments} \label{sec:experiments}

In this section, we experimentally evaluate our proposed algorithms alongside several standard baselines.\footnote{The code can be found at \url{https://github.com/caitree/GoodAction}.}  We first provide a simple proof-of-concept experiment to support our theoretical findings on the lenient regret, but we pay significantly more attention to evaluating the good-action identification algorithms proposed in Section \ref{sec:alg}, since these are designed for practical (rather than theoretical) purposes.

\subsection{Behavior of the Lenient Regret} \label{sec:behavor}

In this experiment (but not later ones), we consider the case of {\em fixed and known} kernel hyperparameters, since our theory assumes this.  
Since the theoretical choice of $\beta_t$ is known to be overly conservative \cite{Sri09}, we manually set $\beta_t^{1/2}=\sqrt{\log(2t)^3}$ in both algorithms.  We fix $|D| = 2500$ points by discretizing $[0,1]^2$ to a $50\times50$ grid.

Figure \ref{fig:ucb_lenient} plots the standard and lenient regret for a 2D synthetic GP function drawn using the SE kernel with parameter $l=0.1$ and $\sigma^{\text{SE}}=1$.  We set the noise level to $\sigma = 0.02$, and the lenient regret parameter as $\Delta = 0.6$, with the latter choice being made in order to form two disjoint regions of good actions.  We see that GP-UCB and the elimination algorithm initially behave similarly, but the lenient regret for the latter completely flattens out by time $700$, whereas the lenient regret GP-UCB only remains gradually increasing, and the standard regret remains more significantly increasing.  This behavior is consistent with Theorems \ref{thm:main1} and \ref{thm:main2}.

We emphasize that elimination crucially depends on having strong prior knowledge of the kernel, hence performing slightly better here.  However, we will see in the following sections that GP-UCB remains effective even without such prior knowledge.


\begin{figure}
    \begin{centering}
        \includegraphics[width=0.375\textwidth]{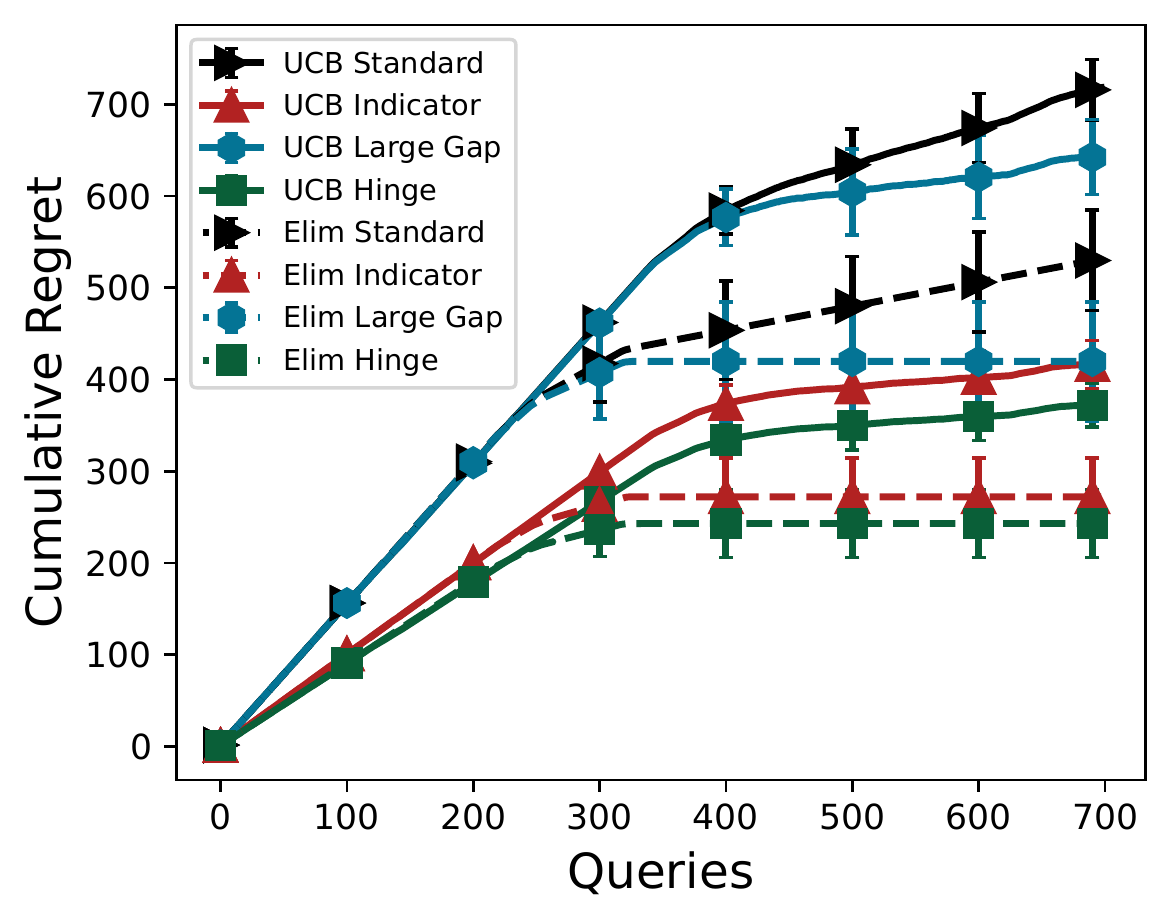}
        \par
    \end{centering}
    \caption{Standard and lenient regret for a 2D synthetic GP. \label{fig:ucb_lenient}} 
    \vspace{-0.1in}
\end{figure}

\subsection{Good-Action Identification Setup} \label{sec:ga_setup}


{\bf GP model.}  We adopt the SE kernel with tunable hyperparameters (lengthscale $l$ and scale $\sigma^{\text{SE}}$).\footnote{The implementation of the GP model comes from \url{https://github.com/ntienvu/MiniBO/}} The hyperparameters are updated every 3 iterations by optimizing the log-likelihood \cite{Ras06} within the range $l \in [10^{-3},1]$ and  $\sigma^{\text{SE}}\in[5\times 10^{-2}, 1.5]$ using the built-in SciPy optimizer based on L-BFGS-B.

{\bf Choice of good-action threshold.} In certain cases, we manually set $\eta$ and specify its value, whereas in other cases, we select $\eta$  such that roughly a fraction $\xi \in (0,1)$ of the domain lies above the threshold.  To do so, we uniformly sample 10,000 actions and take the empirical $\xi$-quantile of their function values.


{\bf Optimization algorithms.} Along with the good-action identification algorithms introduced in Section \ref{sec:alg}, we evaluate the performance of several optimization baselines \cite{Sha16,Wan17}, namely, GP-UCB, PI, EI, Thompson sampling (TS), and MES.  For GP-UCB, we set $\beta_t^{1/2} = \sqrt{\log t}$,\footnote{This is lower than in Section \ref{sec:behavor}, since there we wanted to be confident that the elimination algorithm eliminates correctly.} which we found to provide a suitable exploration/exploitation trade-off.

{\bf Other details.} To simplify the experimental evaluation, we focus primarily on noiseless function evaluations, but a noisy setting will also be considered in Section \ref{sec:noisy}.  We optimize the acquisition functions using the built-in SciPy optimizer with 10 random restarts.  In the case of integer-valued variables, we work on the continuous space and round the decimal to the nearest integer.  



{\bf Evaluation.} Except where stated otherwise, we evaluate the performance by computing the proportion of runs for which a good action was found up to the indicated time.  We perform 25 trials with 10 experiments each, with each experiment generating a fresh random initial set of 3 points to sample (common to all algorithms).  The mean and standard deviation are then computed across trials, with error bars indicating half of a standard deviation.

\subsection{Noiseless Synthetic Functions}


We consider a variety of widely-used synthetic functions whose descriptions can be found at \cite{SFU_Funcs}.  Here the threshold $\eta$ is chosen so that (roughly) a $\xi = \frac{1}{100}$ fraction of points are good; the effect of varying $\xi$ is explored in Appendix \ref{sec:fractions}.  The results are shown in Figure \ref{fig:benchmark_noiseless}.

These experiments indicate that both optimization-based and good-action based algorithms can perform well in terms of finding good actions, but the latter does so slightly faster in these experiments.  In particular, the PG and EG algorithms appear to be most effective.  We believe that GS is slightly slower here due to increased exploration, which may be of less benefit for good-action identification compared to regular optimization.

\begin{figure}[h!]
    \centering
    \begin{subfigure}{\columnwidth}
        \centering
        \includegraphics[width=\sizeiconf]{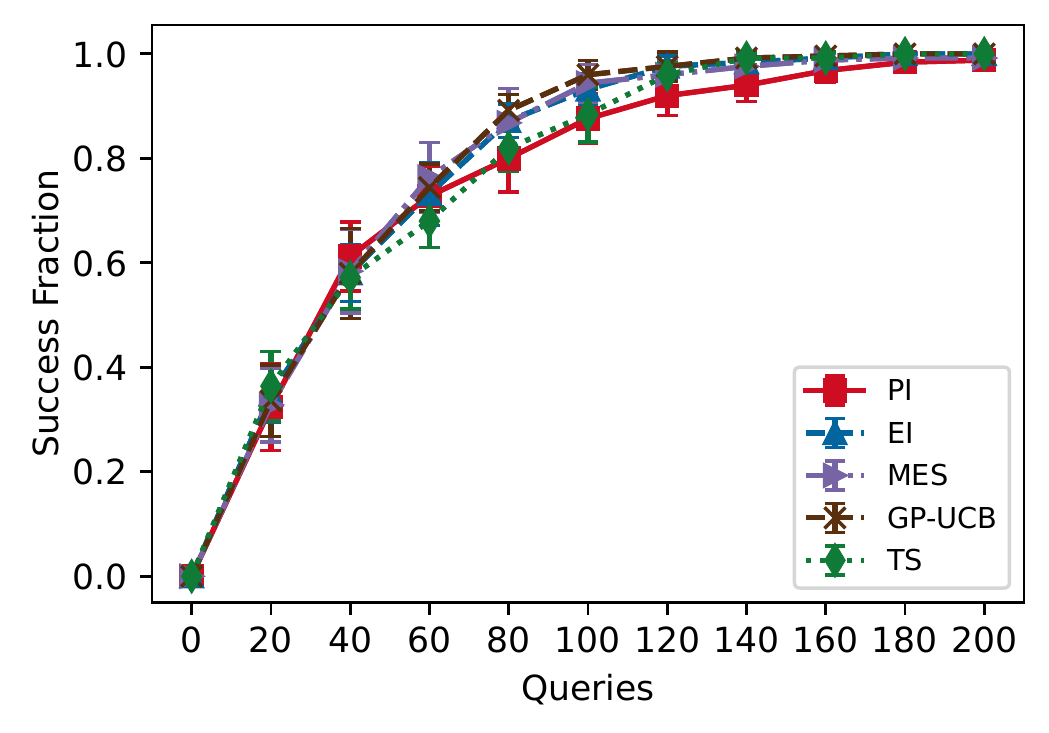}
        \includegraphics[width=\sizeiconf]{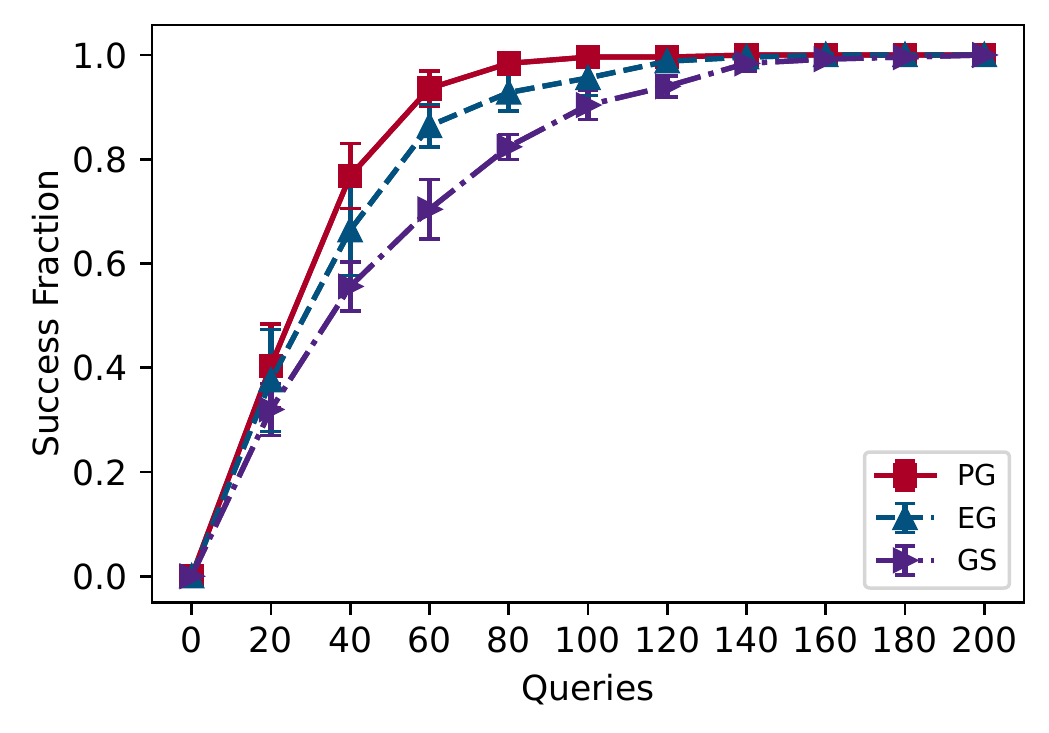}
        \caption{Eggholder 2D}
    \end{subfigure}
    
    
    \begin{subfigure}{\columnwidth}
        \centering
        \includegraphics[width=\sizeiconf]{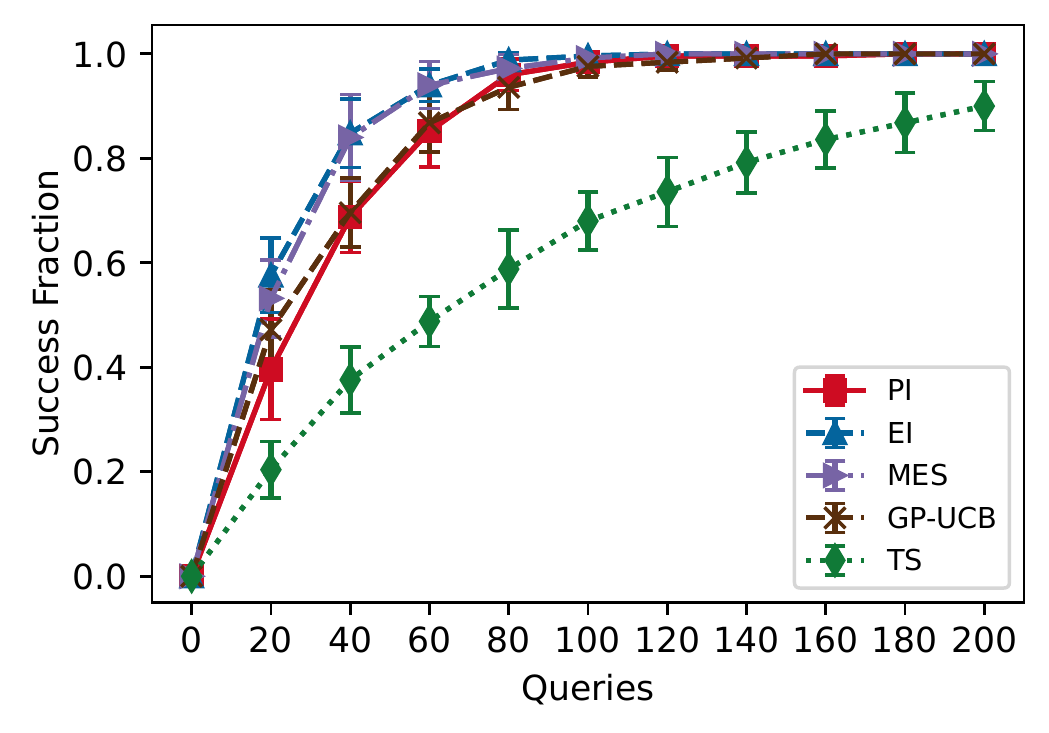}
        \includegraphics[width=\sizeiconf]{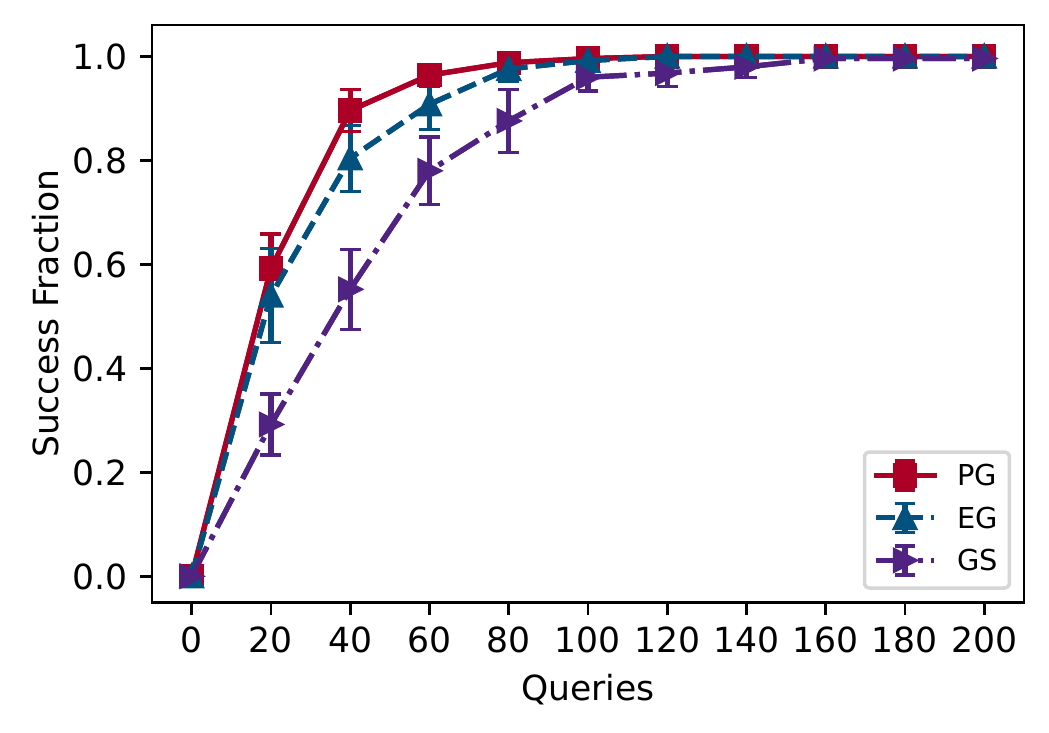}
        \caption{Alpine 6D}
    \end{subfigure}
    
    \caption{Results for noiseless synthetic functions with $\xi = \frac{1}{100}$. \label{fig:benchmark_noiseless}}
\end{figure}

\vspace*{-1ex}
\subsection{Noiseless Non-Synthetic Functions}
\vspace*{-1ex}

{\bf Robot pushing.} We consider the robot pushing objective on a two-dimensional plane from \cite{Wan17}, where the goal is to find a good enough pre-image for pushing an object to a fixed target location $r_g$.  The 3-dimensional function takes robot location $(r_x, r_y)$ and pushing duration $r_t$ as input (the pushing angle is fixed to be $\arctan \frac{r_y}{r_x}$), then outputs the reversed gap between the final location and the target location,  $5 - \|c(r_x, r_y, r_t) - r_g\|$, where $c(\cdot)$ calculates the robot final location.  The 4-dimensional function takes an additional input $r_{\theta}$ specifying the angle to be pushed.  The maximum function value is $5$, and we set $\eta=4.75$.  

{\bf Hyperparameter tuning.}  We consider tuning a regression task using XGBoost \cite{chen16} on the well-known Boston housing dataset.  We perform 3-fold cross-validation, using a fixed seed in order to provide deterministic behavior.  The five parameters that we tune are the maximum tree depth, the learning rate, the maximum delta step for each leaf output, the subsampling ratio of features, and the subsampling ratio of training instances.  We take the objective function to be 10 minus the root-mean-square error (RMSE) on the test fold, and set $\eta=7$. 

{\bf Results.} The results are shown in Figure \ref{fig:cumulative_plot_real}.  We observe similar overall behavior to the above synthetic functions, with PG performing best, and particularly noticeable improvements in the robot pushing experiment.

\begin{figure}
    \centering
    \begin{subfigure}{\columnwidth}
        \centering
        \includegraphics[width=\sizeiconf]{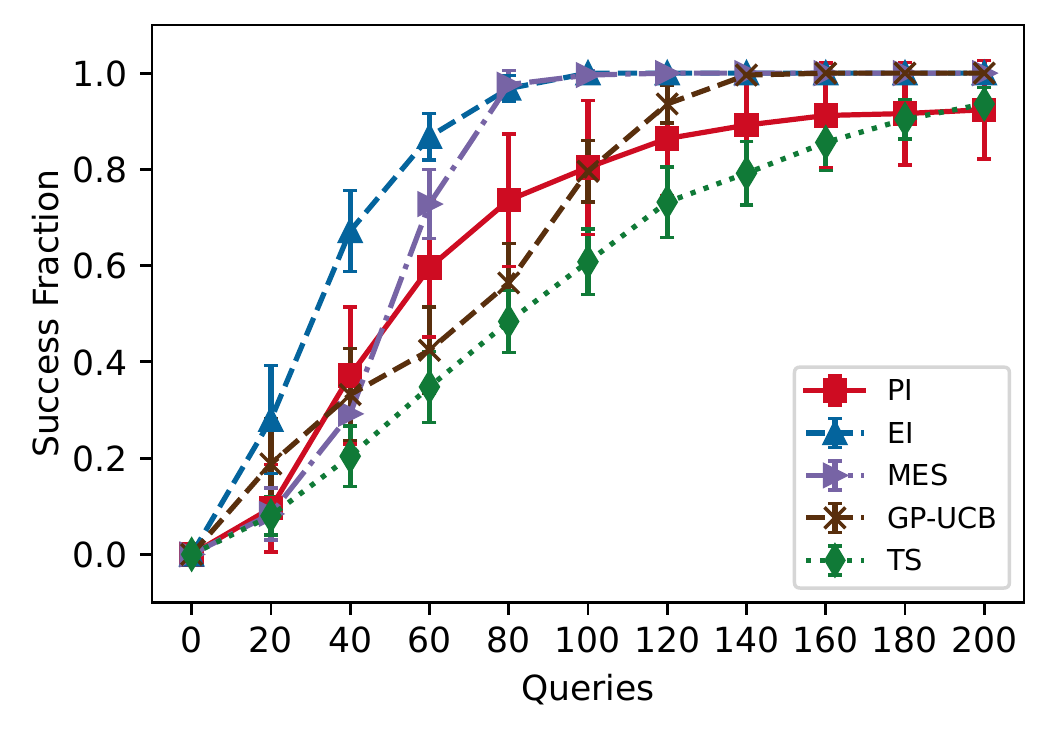}
        \includegraphics[width=\sizeiconf]{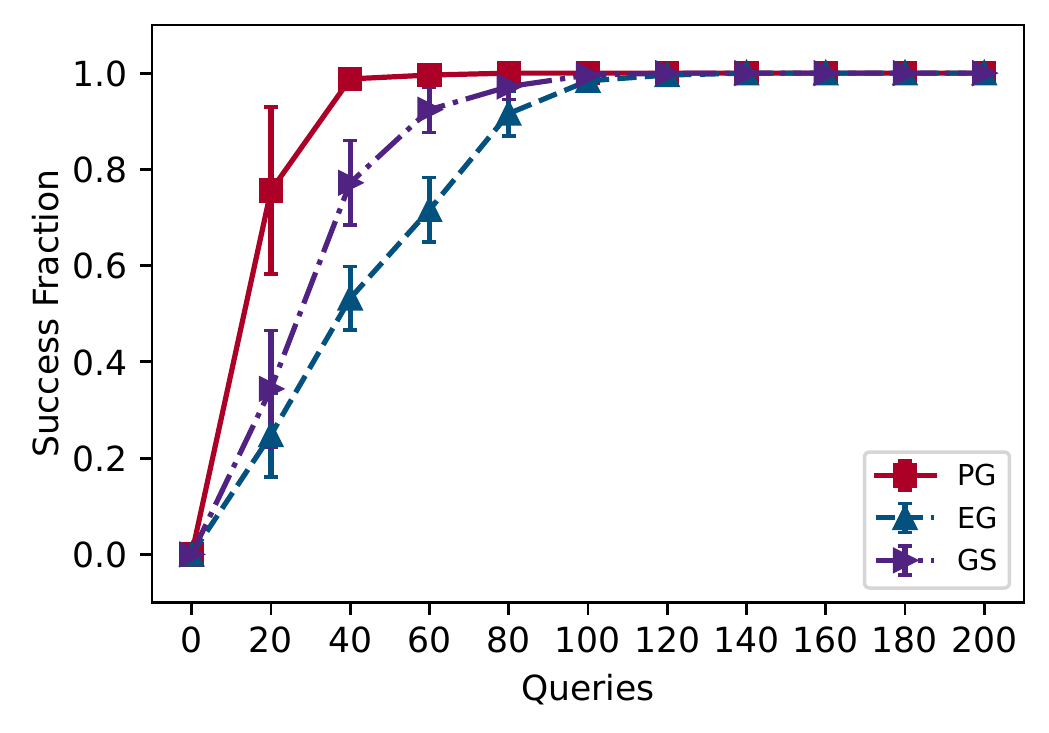}

        \caption{Robot Pushing 3D}
        \label{fig:robot_3}
    \end{subfigure}

    \begin{subfigure}{\columnwidth}
        \centering
        \includegraphics[width=\sizeiconf]{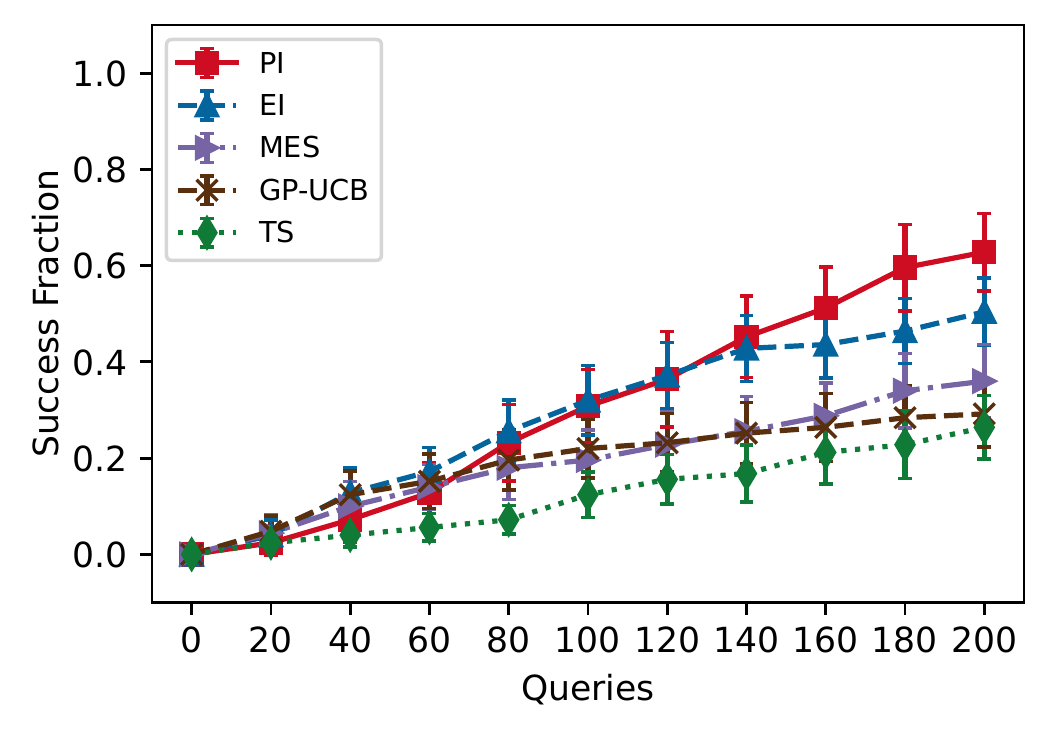}
        \includegraphics[width=\sizeiconf]{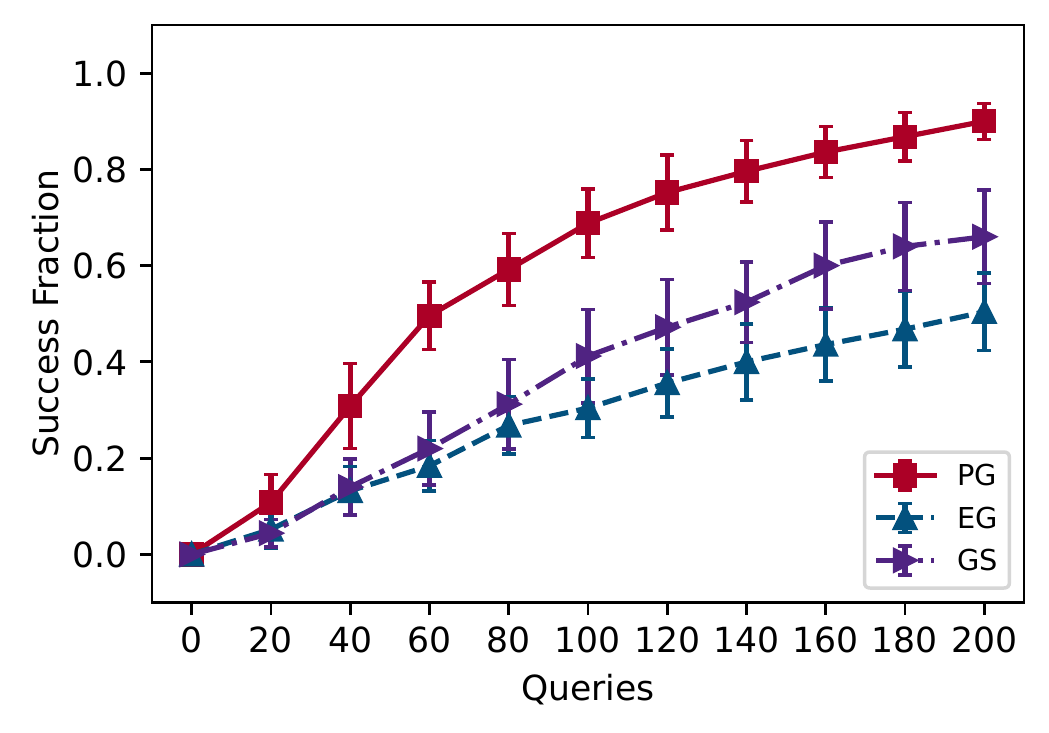}

        \caption{Robot Pushing 4D}
        \label{fig:robot_4}
    \end{subfigure}

    \begin{subfigure}{\columnwidth}
        \centering
        \includegraphics[width=\sizeiconf]{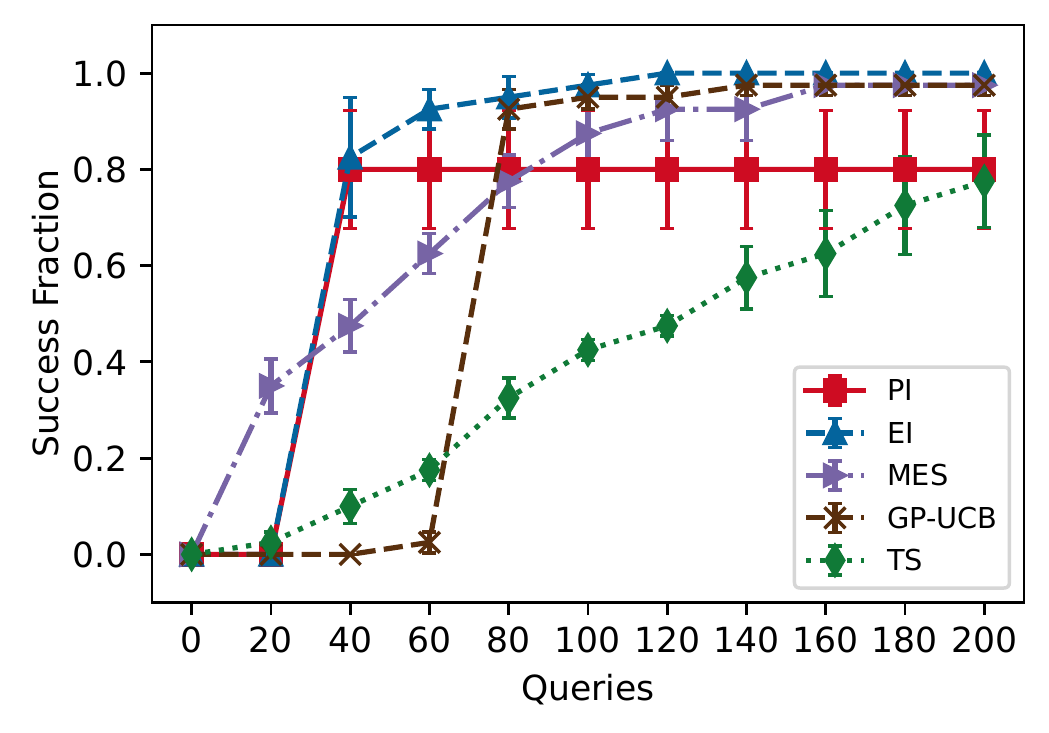}
        \includegraphics[width=\sizeiconf]{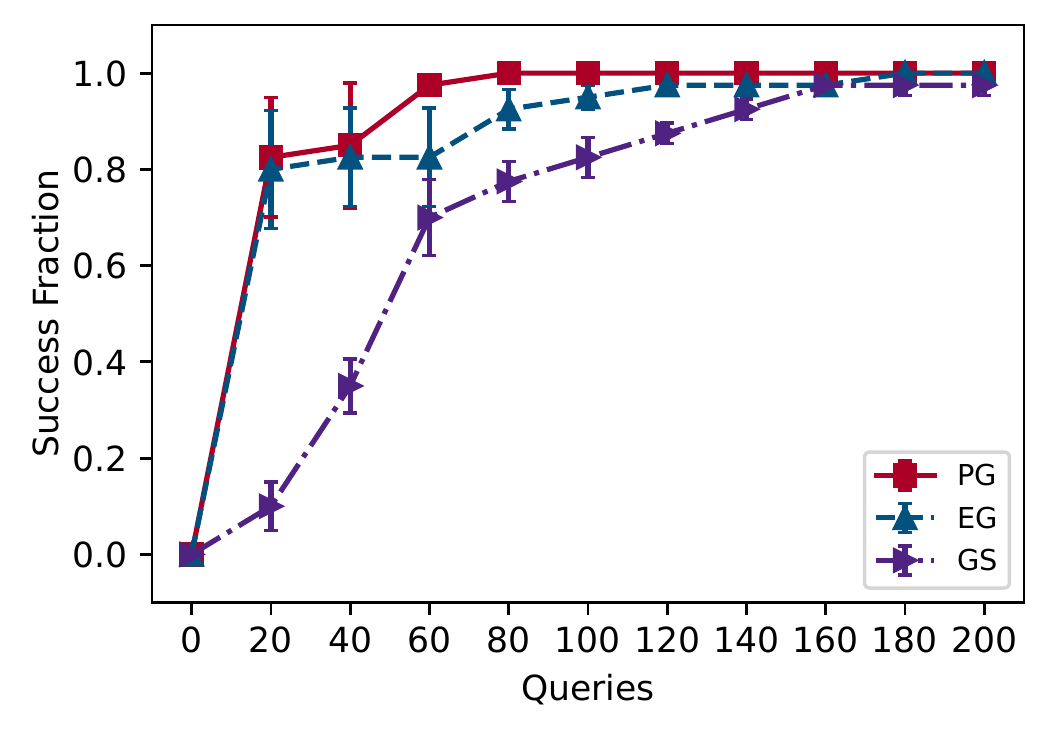}

        \caption{XGBoost Boston}
        \label{fig:xgbb}
    \end{subfigure}

    \caption{Performance comparison on non-synthetic datasets. \label{fig:cumulative_plot_real}}
    
\end{figure}

\subsection{The Effect of Noise} \label{sec:noisy}
In this experiment, we add zero-mean Gaussian noise with standard deviation $\sigma = 0.05$ to each evaluation.  Due to the noise, the algorithm can no longer simply stop when a good action is sampled.  Instead, we continue every algorithm up to the maximum time $T = 200$, and at each time instant, we plot the fraction of runs for which the algorithm's {\em best estimate} is a good action.  We take the best estimate to be the point with the highest posterior mean.  

The results for this setting are shown in Figure \ref{fig:benchmark_noisy}.  Unsurprisingly, the noise makes the curves more erratic overall, and sometimes even non-monotone.  Interestingly, the gains offered by PG are considerable for the Keane function, and also marginally visible for the Ackley function.

\begin{figure}
    \centering
    \begin{subfigure}{\columnwidth}
        \centering
        \includegraphics[width=\sizeiconf]{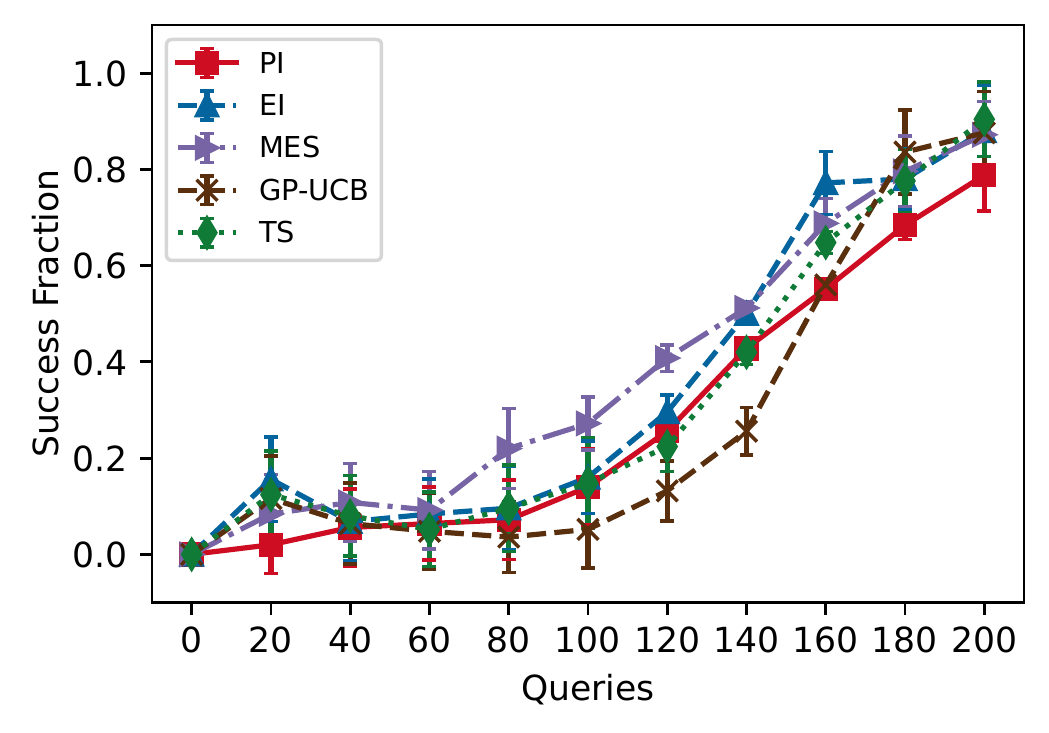}
        \includegraphics[width=\sizeiconf]{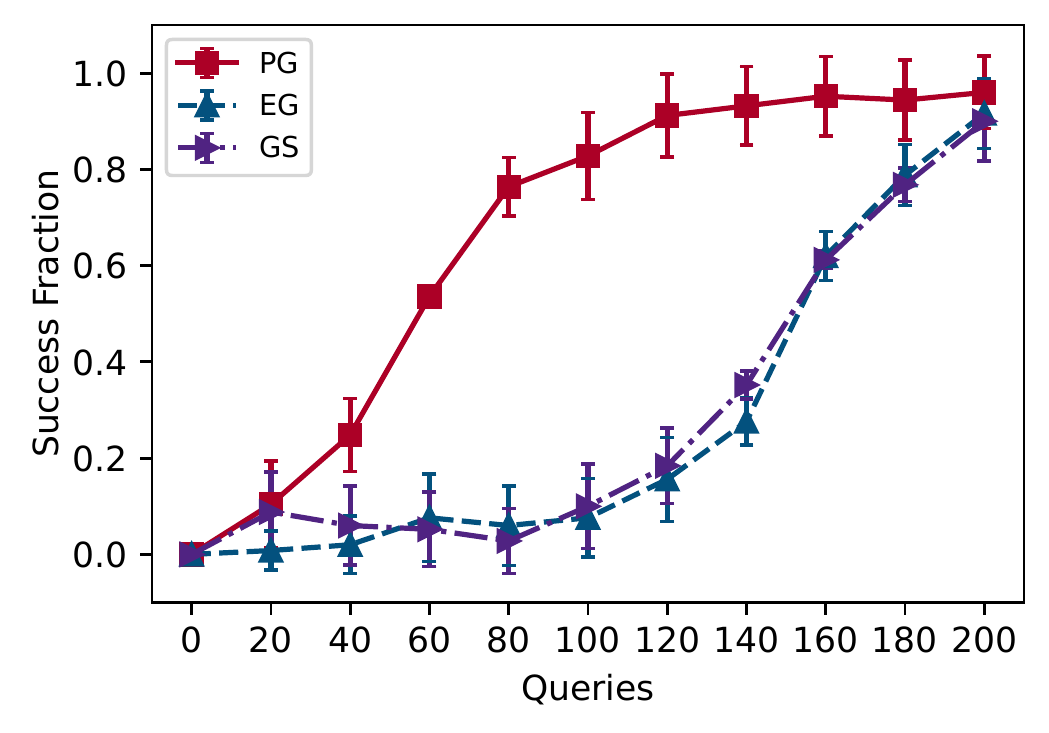}
        
        \caption{Keane}
    \end{subfigure}

    \begin{subfigure}{\columnwidth}
        \centering
        \includegraphics[width=\sizeiconf]{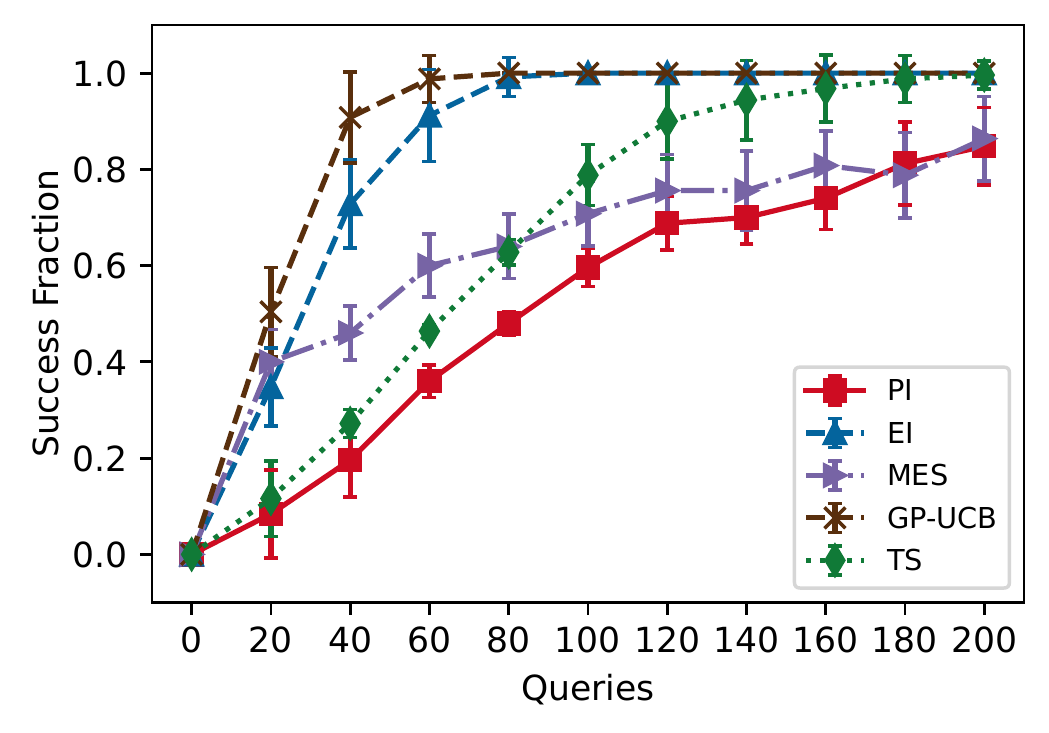}
        \includegraphics[width=\sizeiconf]{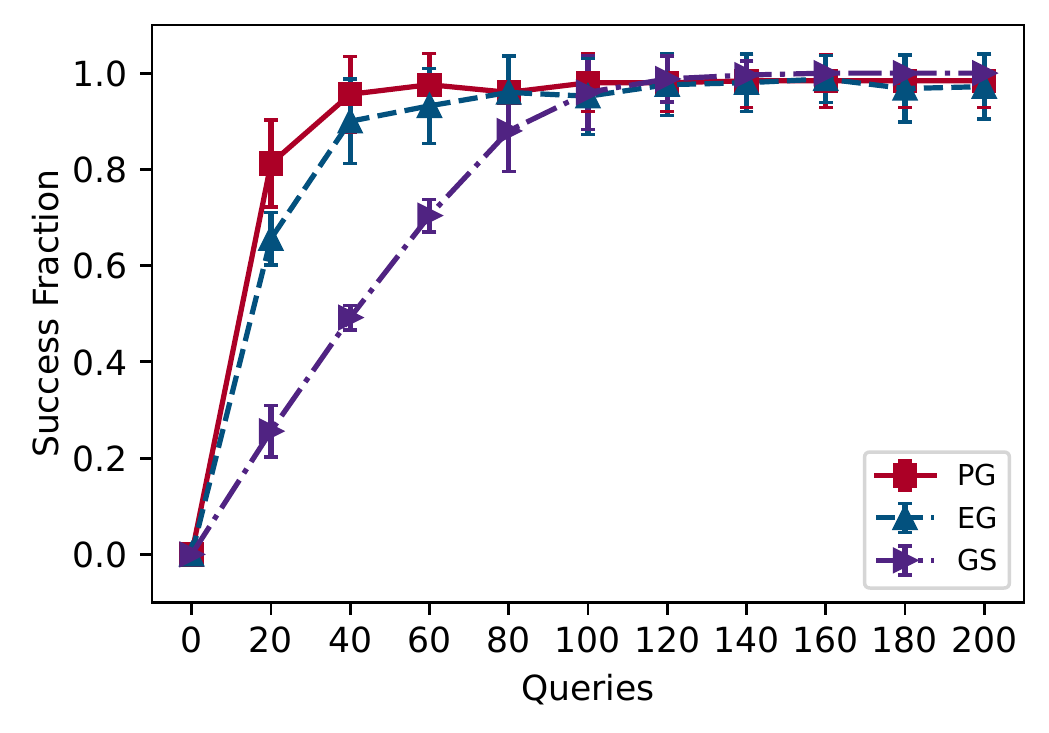}

        \caption{Ackley 6D}
    \end{subfigure}
    \caption{Results for the noisy setting with $\xi = \frac{1}{100}$. \label{fig:benchmark_noisy}}
    
\end{figure}

\subsection{Additional Experiments}

In Appendix \ref{sec:additional}, we provide additional experiments exploring (i) the effect of varying $\eta$ so that the space of good actions grows or shrinks, and (ii) the robustness of our algorithms when no good action exists (i.e., $\eta > f(\xv^*)$).

\subsection{Summary}

Overall, we believe that our experiments indicate PG to be a highly effective algorithm for good-action identification, with EG typically also being competitive.  While GS was typically less effective in the experiments that we ran, it may still be interest to further investigate further when non-myopic methods can help more significantly.


\section{Conclusion} \label{sec:conclusion}

We have established theoretical bounds on the lenient regret for Gaussian process bandits, indicating a significant reduction compared to the standard notion of cumulative regret.  In addition, in the fixed-threshold good-action identification problem, we provided several algorithms that exploit knowledge of the threshold, and provided experimental evidence that PG is particularly effective in practice.

%% file: appendixA.tex
\section{Discussion on GP-UCB with Intersected Confidence Bounds} \label{sec:intersect}

The reason that the lenient regret bounds in Theorem \ref{thm:main1} grow unbounded as $T \to \infty$ is that $\lim_{t \to \infty}\beta_t = \infty$.  For the confidence bounds to remain valid uniformly across time, this appears to be unavoidable.  On the other hand, one may consider preventing the UCB and LCB scores from growing unbounded by using {\em intersected confidence bound}, defined as follows:
\begin{align} 
\ucbbar_t(\xv) &= \min_{t' \le t} \ucb_{t'}(\xv), \label{eq:ucb_int} \\
\lcbbar_t(\xv) &= \max_{t' \le t} \lcb_{t'}(\xv), \label{eq:lcb_int}
\end{align}
with $\ucb_{t'}(\cdot)$ and $\lcb_{t'}(\cdot)$ given in Lemma \ref{lem:conf}.  Since the original confidence bounds hold uniformly across time with high probability, the same is true for these intersected confidence bounds.  We note that this intersecting approach has previously been used in works such as \cite{Sui15,Bog20}.

Unfortunately, we expect that even when the UCB algorithm makes use of $\ucbbar_t(\cdot)$ instead of $\ucb_t(\xv)$, either the lenient regret still grows unbounded as $t \to \infty$, or it is very challenging the prove that it remains bounded.  To understand why we expect such difficulties, consider the scenario in which, in some relatively early round, the UCB score of some bad point $\xv_{\rm bad}$ reaches $f(\xv^*) + \epsilon$ for some extremely small $\epsilon > 0$, and then remains there for a long time due to the intersecting done in \eqref{eq:ucb_int}.  After a long time, points near $\xv^*$ will have been sampled enough times for the UCB scores near $\xv^*$ to fall below $f(\xv^*) + \epsilon$, meaning the algorithm will return to sampling $\xv_{\rm bad}$ (or some similar/nearby point).  However, by this stage, $\beta_t$ may have grown so large that it takes many samples of $\xv_{\rm bad}$ for the UCB score to fall below $f(\xv^*)$, incurring significant regret.

One may envision overcoming this difficulty by showing that the these events of UCB scores falling just above $f(\xv^*)$ (and staying there) are unlikely enough to be incorporated into the overall error probability.  However, this appears to be a highly non-trivial modification to the analysis, and we make no attempt to do so.

Alternatively, following a similar approach \cite{Bog20}, one could multiply by $\beta_t$ by a factor of two in the earlier rounds (e.g., for all $t \le N_{\max}$ with $N_{\max}$ defined in \eqref{eq:Nmax}), then revert to the original choice from Lemma \ref{lem:conf} in the later rounds, while still intersecting the confidence bounds across time.  By doing this, the UCB scores of bad actions that are slightly above $f(\xv^*)$ with the doubled confidence bounds will fall below $f(\xv^*)$ upon halving.  This approach can be used to establish a similar regret bound to that of Theorem \ref{thm:main2}, but it comes with the rather unnatural step of halving the confidence width after a suitably-chosen number of rounds.

Finally, similar to the previous paragraph, one could adopt an {\em explore-then-commit} strategy (e.g., see Chapter 6 of \cite{Csa18}).  While this could provide a bound on the indicator regret similar to Theorem \ref{thm:main2}, the hinge and large-gap regrets would be significantly higher due to typically incurring $\Omega(1)$ regret for each bad action sampled.  Specifically, the dependence on $\Delta$ would be $\frac{1}{\Delta^2}$ instead of the improved $\frac{1}{\Delta}$ appearing in Theorem \ref{thm:main2}.

%% file: appendixBCD.tex
\section{Proofs of Main Results} \label{sec:proofs}

In this section, we prove Theorems \ref{thm:main1}, \ref{thm:main2}, and \ref{thm:main_lb}.  We start with some auxiliary results for the upper bounds.

\subsection{Auxiliary Results}

The analyses of \cite{Sri09} and \cite{Cho17} are based on first bounding the regret in terms of $\sum_{t=1}^T \sigma_{t-1}(\xv_t)$, upper bounding this quantity by $\sqrt{T \sum_{t=1}^T \sigma_{t-1}^2(\xv_t)}$ via Cauchy-Schwartz, and then establishing that $\sum_{t=1}^T \sigma_{t-1}^2(\xv_t) \le O(\lgamma_T)$.  The following lemma gives a useful generalization of the latter statement.

\begin{lem} \label{lem:sum_sampled}
    {\em (Bounding a Sum of Sampled Variances)} For any sequence of sampled points $\xv_1,\dotsc,\xv_T$ and any subset $\Tc \subseteq \{1,\dotsc,T\}$, letting $N = |\Tc|$, we have
    \begin{equation}
        \sum_{t \in \Tc} \sigma_{t-1}^2(\xv_t) \le C_2 \lgamma_N,
    \end{equation}
    where $C_2 = \frac{2\lambda^{-1}}{\log(1+\lambda^{-1})}$.
\end{lem}
\begin{proof}
    Denote the $N$ points indexed by $\Tc$ (i.e., $\{\xv_t\}_{t\in\Tc}$) as $\xvtilde_1,\dotsc,\xvtilde_N$, where the indexing is done in the order that the points were sampled.  For $i=1,\dotsc,N$, let $\sigmatil_{i}^2(\xv)$ be the (hypothetical) GP posterior variance that would arise from sampling $\xvtilde_1,\dotsc,\xvtilde_i$ alone (note that posterior variance only depends on the sampled locations, not the observations \cite{Ras06}).  It is well-known from \cite{Sri09} that $\sum_{i=1}^N \sigmatil_{i-1}^2(\xvtilde_i) \le C_2 \lgamma_N$, so we only need to show that $\sum_{t \in \Tc} \sigma_{t-1}^2(\xv_t) \le \sum_{i=1}^N \sigmatil_{i-1}^2(\xvtilde_i)$.  Indexing the entries of $\Tc$ in order by $t_1,\dotsc,t_N$, the latter claim in turn holds as long as $\sigma_{t_i-1}^2(\xv_{t_i}) \le \sigmatil_{i-1}^2(\xvtilde_i)$ for all $i=1,\dotsc,N$.
    
    By definition, $\xv_{t_i}$ is precisely $\xvtilde_i$.  Moreover, the posterior variance $\sigma_{t_i-1}^2(\cdot)$ is computed using $t_i-1$ sampled points, $i-1$ of which are $\xvtilde_1,\dotsc,\xvtilde_{i-1}$.  In contrast, $\sigmatil_{i-1}^2(\cdot)$ is computed based on $\xvtilde_1,\dotsc,\xvtilde_{i-1}$ alone.  Since adding points to the set of sampled points cannot increase the posterior variance in a GP model \cite{Ras06}, the desired claim $\sigma_{t_i-1}^2(\xv_{t_i}) \le \sigmatil_{i-1}^2(\xvtilde_i)$ follows, and the proof is complete.
\end{proof}

\subsection{Bounding the Number of Bad Actions for GP-UCB}

Let $\Tcbad$ denote the set of times at which GP-UCB chooses a bad action, and let $N = |\Tcbad|$.  By Lemma \ref{lem:sum_sampled}, we have
\begin{equation}
    \frac{1}{N} \sum_{t \in \Tcbad} \sigma_{t-1}^2(\xv_t) \le \frac{C_2 \lgamma_N}{N},
\end{equation}
where we multiplied by $\frac{1}{N}$ on both sides for convenience.  Since the minimum is upper bounded by the average, it follows that
\begin{equation}
    \min_{t \in \Tcbad} \sigma_{t-1}^2(\xv_t) \le \frac{C_2 \lgamma_N}{N}. \label{eq:min_avg}
\end{equation}
Now, letting $\tau$ denote the time index attaining the minimum in \eqref{eq:min_avg}, and supposing that the high-probability confidence bound event in Lemma \ref{lem:conf} holds, we have
\begin{align}
    \ucb_{\tau}(\xv_{\tau}) 
    &= \lcb_{\tau}(\xv_{\tau}) + 2\beta_{\tau}^{1/2}\sigma_{\tau-1}(\xv_{\tau}) \label{eq:tau1}  \\
    &\leq f(\xv_{\tau}) + 2\beta_{\tau}^{1/2}\sigma_{\tau-1}(\xv_{\tau}) \label{eq:tau2} \\
    &\leq f(\xv^*) -\Delta + 2\beta_{\tau}^{1/2}\sigma_{\tau-1}(\xv_{\tau}) \label{eq:tau3} \\
    &\leq f(\xv^*) + 2\beta_T^{1/2}\sqrt{\frac{C_2\lgamma_N}{N}} - \Delta \label{eq:tau4} \\
    &\leq \ucb_{\tau}(\xv^*) + \sqrt{\frac{C_1\beta_T\lgamma_N}{N}} - \Delta, \label{eq:tau5} 
\end{align}
where:
\begin{itemize}[leftmargin=5ex,itemsep=0ex,topsep=0.25ex]
    \item \eqref{eq:tau1} follows since the upper and lower confidence bounds differ by $2\beta_{\tau}^{1/2}\sigma_{\tau-1}(\xv_{\tau})$;
    \item \eqref{eq:tau2} and \eqref{eq:tau5} follow from the validity of the confidence bounds, and the latter also defines $C_1 = 4C_2$;
    \item \eqref{eq:tau3} follows since $f(\xv_{\tau}) \le f(\xv^*) - \Delta$ due to $\xv_{\tau}$ being a bad point;
    \item \eqref{eq:tau4} applies \eqref{eq:min_avg}, along with $\beta_{\tau} \le \beta_{T}$ due to monotonicity.
\end{itemize} 
Since $\xv_{\tau}$ is the point at time $\tau$ with the highest UCB score by definition, we observe from \eqref{eq:tau5} that we must have $\sqrt{\frac{C_1\beta_T\lgamma_N}{N}} - \Delta \ge 0$ in order to avoid a contradiction.  Re-arranging, we obtain the equivalent condition
\begin{equation}
    N\leq \frac{C_1\lgamma_N\beta_T}{\Delta^2}. \label{eq:N_bound}
\end{equation}
Since this was proved only assuming the validity of the confidence bounds in Lemma \ref{lem:conf}, which in turn holds with probability at least $1-\delta$, the claim on $\Rtilde_T^{\rm ind}$ in Theorem \ref{thm:main1} follows.

\subsection{Bounding the Large Gap Regret for GP-UCB}

Since $\Rtilde_T^{\rm hinge} \le \Rtilde_T^{\rm gap}$ (see Figure \ref{fig:regret_compare}), it suffices to upper bound $\Rtilde_T^{\rm gap}$.  We first write
\begin{equation}
    \Rtilde_T^{\rm gap} = \sum_{t=1}^{T} r_t\cdot\openone (r_t > \Delta) = \sum_{t\in\Tc_{\rm bad}}r_t. \label{eq:gap_regret}
\end{equation}
Following the steps of \cite{Sri09}, and again conditioning on the validity of the confidence bounds in Lemma \ref{lem:conf}, we have
    \begin{align}
        r_t &= f(\xv^*) - f(\xv_t) \label{eq:gap1} \\
        &\leq \ucb_{t}(\xv^*) - \lcb_{t}(\xv_t) \label{eq:gap2} \\
        &= \ucb_{t}(\xv^*) - \ucb_{t}(\xv_t) + 2\beta_{t}^{1/2}\sigma_{t-1}(\xv_{t}) \label{eq:gap3} \\
        &\leq  2\beta_t^{1/2}\sigma_{t-1}(\xv_t), \label{eq:gap4} 
    \end{align}
    where \eqref{eq:gap2} uses the confidence bounds, \eqref{eq:gap3} follows since the upper and lower confidence bounds differ by $2\beta_{t}^{1/2}\sigma_{t-1}(\xv_{t})$, and \eqref{eq:gap4} uses the fact that $\xv_t$ is the point with the highest UCB score.

    Summing \eqref{eq:gap4} over $t \in \Tcbad$, upper bounding $\beta_t \le \beta_T$, and applying the Cauchy-Schwartz inequality, we obtain
    \begin{equation}
        \Rtilde_T^{\rm gap} \le \sqrt{ 4\beta_T |\Tcbad| \sum_{t \in \Tcbad} \sigma_{t-1}^2(\xv_t)}.
    \end{equation}
    Again letting $N = |\Tcbad|$ denote the number of bad points selected, it follows from Lemma \ref{lem:sum_sampled} that
    \begin{equation}
        \Rtilde_T^{\rm gap} \le \sqrt{ C_1 \beta_T N \lgamma_{N}}.
    \end{equation}
    Since we already established that $N$ satisfies \eqref{eq:N_bound} when the confidence bounds are valid, we can further bound
    \begin{equation}
        \Rtilde_T^{\rm gap} \le \frac{ C_1 \beta_T \lgamma_{N} }{\Delta}.
    \end{equation}
    The bound on $\Rtilde_T^{\rm gap}$ in Theorem \ref{thm:main1} follows by substituting $N \le N_{\max}$ and using the monotonicity of $\lgamma_N$.

\subsection{Bounding the Number of Bad Actions for the Elimination Algorithm}

Our analysis uses similar ingredients as in \cite{Con13,Bog16a,Sri09}.  We first note the well-known fact that as long as the confidence bounds in Lemma \ref{lem:conf} are valid, the algorithm never eliminates $\xv^*$.  This is because having the UCB of $\xv^*$ be below another point's LCB would contradict the optimality of $\xv^*$.

Suppose that the elimination algorithm has run up to some number of rounds $N$.  Using Lemma \ref{lem:sum_sampled} with $\Tc = \{1,\dotsc,N\}$, we have
\begin{equation}
    \frac{1}{N} \sum_{t =1}^N \sigma_{t-1}^2(\xv_t) \le \frac{C_2 \lgamma_N}{N}, \label{eq:summands}
\end{equation}
where we again divided both sides by $N$ for convenience.  Using the standard property that the GP posterior variance always decreases as more points are selected, and noting the algorithm chooses the point with the highest variance, we find that $\sigma_{N-1}^2(\xv_N)$ is the smallest summand in \eqref{eq:summands}, and hence
\begin{equation}
    \sigma_{N-1}^2(\xv_N) \le \frac{C_2 \lgamma_N}{N}.
\end{equation}
Moreover, since $\xv_N$ is defined to maximize $\sigma_{N-1}^2(\cdot)$, it follows that
\begin{equation}
    \max_{\xv \in M_{N-1}} \sigma_{N-1}^2(\xv) \le \frac{C_2 \lgamma_N}{N}. \label{eq:var_bound}
\end{equation}
That is, all non-eliminated points have posterior variance at most $\frac{C_2 \lgamma_N}{N}$ after time $N$.

We now fix an arbitrary non-eliminated bad point $\xv_{\rm bad}$, and note the following analogous steps to \eqref{eq:tau1}--\eqref{eq:tau5} (whose explanations are similar and thus mostly omitted):
\begin{align}
    \ucb_{N}(\xv_{\rm bad})
    &= \lcb_{N}(\xv_{\rm bad}) + 2\beta_{N}^{1/2}\sigma_{N-1}(\xv_{\rm bad}) \label{eq:elim1}  \\
    &\leq f(\xv_{\rm bad}) + 2\beta_{N}^{1/2}\sigma_{N-1}(\xv_{\rm bad}) \label{eq:elim2} \\
    &\leq f(\xv^*) -\Delta + 2\beta_{N}^{1/2}\sigma_{N-1}(\xv_{\rm bad}) \label{eq:elim3} \\
    &\leq \lcb_{N}(\xv^*) - \Delta + 2\beta_{N}^{1/2}\sigma_{N-1}(\xv_{\rm bad}) + 2\beta_{N}^{1/2}\sigma_{N-1}(\xv^*) \label{eq:elim4} \\
    &\leq \lcb_{N}(\xv^*) + 2\sqrt{\frac{C_1\beta_N\lgamma_N}{N}} - \Delta, \label{eq:elim5} 
\end{align}
where \eqref{eq:elim5} applies \eqref{eq:var_bound} for both $\xv \in \{\xv_{\rm bad}, \xv^*\}$.

Since \eqref{eq:elim5} applies to an arbitrary non-eliminated bad point, we find that in order for any bad points to remain non-eliminated after time $N$, it must be the case that $2\sqrt{\frac{C_1\beta_T\lgamma_N}{N}} - \Delta \ge 0$, or equivalently,
\begin{equation}
    N\leq \frac{4C_1\lgamma_N\beta_N}{\Delta^2}. \label{eq:N_bound2}
\end{equation}
In other words, all bad points are eliminated after time $N'_{\max}$, with $N'_{\max}$ defined in \eqref{eq:Nmax'}.  This proves the first part of Theorem \ref{thm:main2}.

\subsection{Bounding the Large Gap Regret for the Elimination Algorithm}

While we performed the analysis leading to \eqref{eq:N_bound2} considering the number of pulls of $\Delta$-suboptimal points, we can similarly replace $\Delta$ by any positive value $\Deltatil$ and reach a similar conclusion.  In the following, it is more convenient to rephrase \eqref{eq:N_bound2} by expressing $\Delta$ in terms of $N$ as $\Delta \le \sqrt{\frac{4C_1\lgamma_N\beta_N}{N}}$.  Replacing $\Delta$ by a generic value of $\Deltatil$, and replacing $N$ by a generic time index $t$, it follows that after $t$ iterations, all non-eliminated arms have regret upper bounded by $\Deltatil_t$, where
\begin{equation}
    \Deltatil_t =  \sqrt{\frac{4C_1\lgamma_t\beta_t}{t}}.
\end{equation}
To bound the large gap regret, we simply sum the regret over all time indices up to $N'_{\max}$, after which we already know from the above analysis that no further (lenient) regret is incurred.  We additionally treat $t=1$ as a special case, noting that the regret incurred is at most $2B$ since $\|f\|_k \le B$ (and thus $|f(\xv)| \le B$ for all $\xv$), yielding
\begin{align}
    \Rtilde_T^{\rm gap} 
        &\le 2B + \sum_{t=2}^{N'_{\max}} \Deltatil_{t-1} \\
        &\le 2B + \sum_{t=1}^{N'_{\max}} \Deltatil_{t} \\
        &\le 2B + \sum_{t=1}^{N'_{\max}} \sqrt{\frac{4C_1\lgamma_t\beta_t}{t}} \label{eq:gap_elim3} \\
        &\le 2B + \sqrt{4C_1\lgamma_{N'_{\max}}\beta_{N'_{\max}}} \sum_{t=1}^{N'_{\max}} \frac{1}{\sqrt t} \label{eq:gap_elim4} \\
        &\le 2B + 4\sqrt{C_1 N'_{\max} \lgamma_{N'_{\max}}\beta_{N'_{\max}}}, \label{eq:gap_elim5}
\end{align}
where \eqref{eq:gap_elim3} uses the definition of $\Deltatil_t$, \eqref{eq:gap_elim4} uses the monotonicity of $\lgamma_t$ and $\beta_t$, and \eqref{eq:gap_elim5} uses the fact that $\sum_{t=1}^N \frac{1}{\sqrt t} \le 2\sqrt{N}$.  Finally, by definition in \eqref{eq:Nmax'}, we have $N'_{\max} \leq \frac{4C_1\lgamma_{N'_{\max}}\beta_{N'_{\max}}}{\Delta^2}$, and substituting into \eqref{eq:gap_elim5} yields $\Rtilde_T^{\rm gap} \le 2B + \frac{8C_1 \lgamma_{N'_{\max}}\beta_{N'_{\max}}}{\Delta}$, as desired.

\subsection{Proofs of the Lower Bounds}

Since our lower bounds follow in a fairly straightforward manner from the analysis in \cite{Cai20}, we do not attempt to give a self-contained analysis (which would require considerable repetition with \cite{Sca17a,Cai20}), and instead only state the differences.

The analysis depends on a parameter $\epsilon > 0$ that is initially arbitrary, and that we will set differently to \cite{Cai20} to account for the different regret notion.  A {\em hard subset} of functions $\{f_1,\dotsc,f_M\} \in \Fc_k(B/3)$ is constructed in a manner such that any given action $x \in D$ is $\epsilon$-optimal for at most one function.  It is shown in \cite{Sca17a} that such a subset exists with the following choices of $M$ depending on the kernel:
\begin{itemize}[leftmargin=5ex,itemsep=0ex,topsep=0.25ex]
    \item For the SE kernel, we can set 
    \begin{equation}
        M = \Bigg\lfloor \Bigg( \frac{ c_1 \sqrt{\log\frac{B (2\pi l^2)^{d/4}}{\epsilon}} }{l} \Bigg)^d \Bigg\rfloor, \label{eq:M_SE}
    \end{equation}
    where $c_1$ is a universal positive constant, and $l$ denotes the length-scale.
    \item For the Mat\'ern kernel, we can set 
    \begin{equation}
        M = \Big\lfloor \Big( \frac{B c_3}{\epsilon} \Big)^{d/\nu} \Big\rfloor, \label{eq:M_Matern}
    \end{equation}
    where $c_3 :=  \big( \frac{1}{\zeta} \big)^{\nu} \cdot \big( \frac{ c_2^{-1/2} }{ 2 (8\pi^2)^{(\nu + d/2)/2} } \big)$, and where $\zeta > 0$ and $c_2 > 0 $ are constants.
\end{itemize}
Once the existence of this function class is established, the analysis in \cite{Cai20} shows that there exists a function $f \in \Fc_k(B)$ and constant $c_0$ such that when the time horizon satisfies 
\begin{equation}
    T < \frac{(M-1)\sigma^2}{2c_0 \epsilon^2} \log\frac{1}{2.4\delta}, \label{eq:T_existing}
\end{equation}
it must hold with probability at least $\delta$ that $\epsilon$-suboptimal actions are selected in at least $\frac{T}{2}$ rounds.

We now turn to the part of the analysis that differs from \cite{Cai20}.  We first use the trivial fact that the cumulative regret up to time $T$ is lower bounded by that up to any $\Ttilde \le T$.  We consider $\Ttilde$ being slightly below the threshold in \eqref{eq:T_existing} (or capped to $T$):
\begin{equation}
    \Ttilde = \min\bigg\{T, \frac{M\sigma^2}{4c_0 \epsilon^2} \log\frac{1}{2.4\delta}\bigg\},  \label{eq:Ttilde}
\end{equation}
and since this choice is smaller than the right-hand side of \eqref{eq:T_existing}, we know that $\epsilon$-suboptimal actions must be played at least $\frac{\Ttilde}{2}$ times.  

To lower bound the lenient regret in the case that $\Phi = \Phi^{\rm ind}$, we simply set $\epsilon = \Delta$, so that being $\epsilon$-suboptimal is exactly equivalent to being a bad action.  In this case, the desired lower bounds follow directly by substituting \eqref{eq:M_SE} and \eqref{eq:M_Matern} into \eqref{eq:Ttilde} and lower bounding the lenient regret by $\frac{\Ttilde}{2}$.  Note that the assumption $\frac{\Delta}{B} = O(1)$ (with a small enough implied constant) implies that \eqref{eq:M_SE} and \eqref{eq:M_Matern} scale as $\Theta\big(\big( \log\frac{B}{\epsilon} \big)^{d/2} \big)$ and $\Theta\big( \big(\frac{B}{\Delta}\big)^{d/\nu} \big)$ respectively.

To lower bound the lenient regret in the case that $\Phi = \Phi^{\rm hinge}$, we notice from the definition of the hinge function that if a $2\Delta$-suboptimal point is selected, then the contribution to the lenient regret is still at least $\Delta$.  Hence, the desired lower bounds follow by setting $\epsilon = 2\Delta$, substituting \eqref{eq:M_SE} and \eqref{eq:M_Matern} into \eqref{eq:Ttilde}, and lower bounding the lenient regret by $\frac{\Ttilde\Delta}{2}$.  Finally, the inequality $\Rtilde_T^{\rm gap} \ge \Rtilde_T^{\rm hinge}$ is trivial by definition (see Figure \ref{fig:regret_compare}).

\section{Additional Good-Action Identification Algorithms} \label{sec:other}

\subsection{Satisficing Thompson Sampling (STS)}\label{sec:sts}

Thompson sampling (TS) samples actions randomly according to the posterior probability of being optimal \cite{Rus18a}.  To adapt TS to the good-action identification problem, we follow an idea proposed in \cite{Rus18} for multi-armed bandits, termed {\em satisficing Thompson sampling} (STS).  In the finite-arm setting, the STS approach samples according to the probability of being the good arm {\em with the lowest index}.  

In our continuous-domain setting, there is no natural order over the arms, so we instead consider the following natural analog: Seek the good action {\em closest to some fixed point $\xv^{\rm c}$} (with the default value being the domain center).  The resulting algorithm is as follows:
\begin{itemize}
    \item Let $\tilde{f}_t$ be a sample from the GP posterior distribution given the first $t-1$ observations;
    \item Choose $\xv_t$ to maximize the following acquisition function:
    \begin{equation}
        \alpha^{\rm STS}_t(\xv) = 
        \begin{cases}
            -\|\xv - \xv^{\rm c}\| & \tilde{f}_t(\xv) \geq \eta \\
            -\infty & {\rm otherwise}.
        \end{cases}
        \label{eq:sts_acq}
    \end{equation}
\end{itemize}
It may be that none of the points in the domain satisfy $ \tilde{f}_t(\xv) \ge \eta$, in which case we simply let $\xv_t$ be a maximizer of $\tilde{f}_t$ (i.e., revert to regular TS).

This approach is primarily suited to scenarios where prior knowledge is available on the approximate location of the maximizer or a good region (captured by $\xv^{\rm c}$).  Since such knowledge is typically unavailable, we only investigate STS in some proof-of-concept experiments here; further studies of TS-type methods for good-action identification is left for future work.  The experimental details are as described in Section \ref{sec:experiments}, and the results shown in Figure \ref{fig:sts}.  


For the Dropwave function the optimal action is precisely at the domain center ($\xv^* = \bzero$), and accordingly, STS performs much better than the other methods.  For the Keane function it is near the center ($\xv^* = (1.39, 0)\ \text{or}\ (0,1.39)$), and STS remains competitive with PG.  Finally, when we shift the Dropwave function so that the good actions are near the boundary ($\xv^* = (-5.12, 5.12)$), we find that STS performs significantly worse.  Thus, these experiments provide evidence that prior knowledge of an approximate function maximizer (or at least a ``good region'') is important for our version of STS to perform well.

\begin{figure}[h!]
    \centering
    \begin{subfigure}{0.325\columnwidth}
        \includegraphics[width=\linewidth]{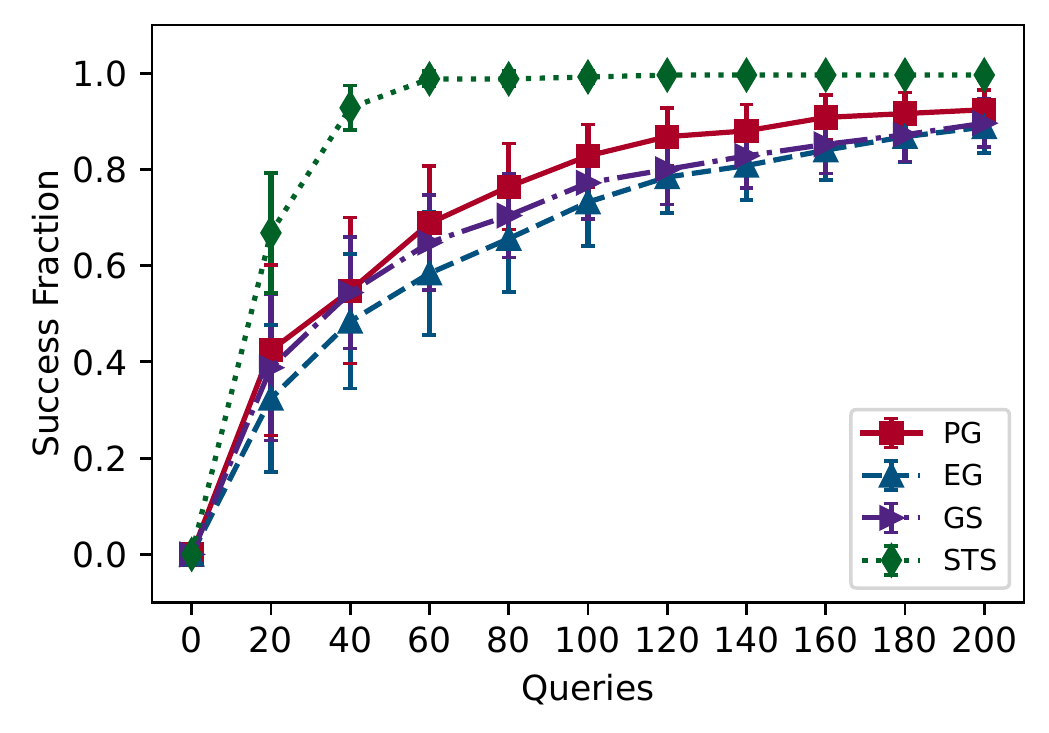}
        \caption{Dropwave}
    \end{subfigure}
    \begin{subfigure}{0.325\columnwidth}
        \includegraphics[width=\linewidth]{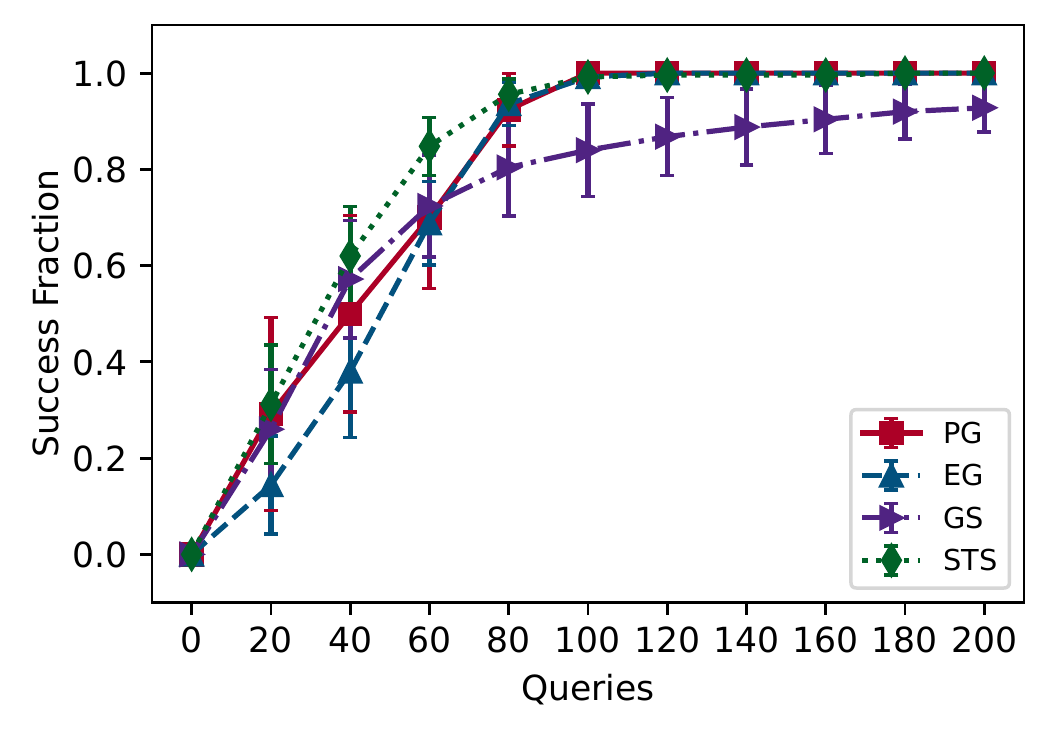}
        \caption{Keane}
    \end{subfigure}
    \begin{subfigure}{0.325\columnwidth}
        \includegraphics[width=\linewidth]{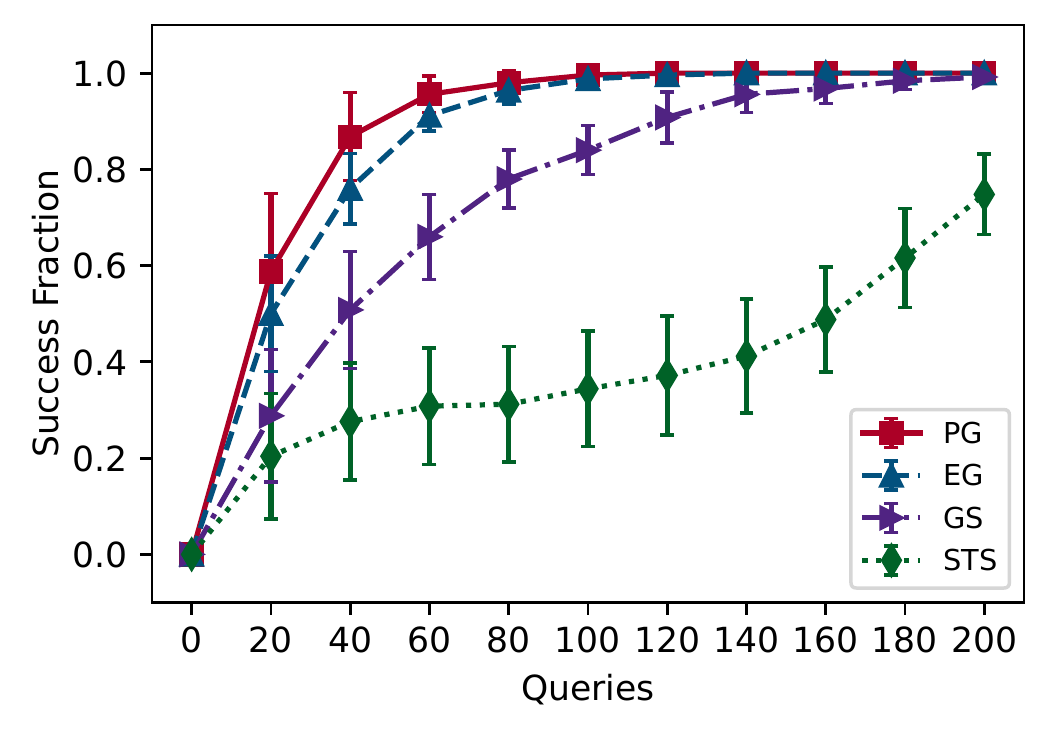}
        \caption{Shifted Dropwave}
    \end{subfigure}

    \caption{Experimental results for good-action identification with Satisficing Thompson Sampling (STS). \label{fig:sts}}
    
\end{figure}

\subsection{Elimination Algorithm} \label{sec:elim_good}

We briefly mention that one can modify the elimination algorithm described in Section \ref{sec:elim} by eliminating all actions whose UCB score is below $\eta$, rather than those whose UCB is below the highest LCB.  That is, we modify \eqref{eq:Mt} as follows:  
\begin{equation}
    M_t = \big\{ \xv \in M_{t-1} \,:\, \ucb_t(\xv) \ge \eta \big\}. \label{eq:Mt2}
\end{equation}
At the times of primary interest where no good action has been found yet, $\eta$ will typically be significantly above the highest LCB score, and hence, more bad actions will be eliminated earlier compared to when using \eqref{eq:Mt}.  However, as discussed in Section \ref{sec:lenient_elim}, elimination algorithms are susceptible to complete failure under kernel misspecification, and we thus do not include this approach in our experiments, in which the kernel hyperparameters are learned online.

\section{Additional Experiments} \label{sec:additional}

Here we present further experiments for good-action identification, adopting the same setup as described in Section \ref{sec:ga_setup} except where stated otherwise.

\subsection{Comparison of Different Threshold Values} \label{sec:fractions}


We explore the effect of varying $\eta$ using the Ackley function and the robot pushing function.  For the Ackley function, we consider choosing $\eta$ such that roughly a fraction $\xi$ of points are good, as detailed in Section \ref{sec:ga_setup}.  The results for $\eta \in \big\{ \frac{1}{400}, \frac{1}{100}, \frac{1}{50}\big\}$ are shown in Figure \ref{fig:a6_for_diff_k}.  For the robot pushing objective, we choose $\eta \in \big\{4.0, 4.5, 4.75\big\}$, and the results are shown in Figures \ref{fig:robot3_for_diff_k} and \ref{fig:robot4_for_diff_k} (3D and 4D versions, respectively).

In each experiment, we observe fairly similar behavior for each good-action threshold, but we find that increasing $\xi$ (or equivalently, decreasing $\eta$) naturally makes all algorithms find good points faster.  A somewhat less obvious finding is that this also tends to bring all of the curves closer together, suggesting that most ``reasonable'' algorithms can quickly find a good action when sufficiently many of them exist.

\begin{figure}
    \centering
    \begin{subfigure}{\columnwidth}
        \centering
        \includegraphics[width=0.3\columnwidth]{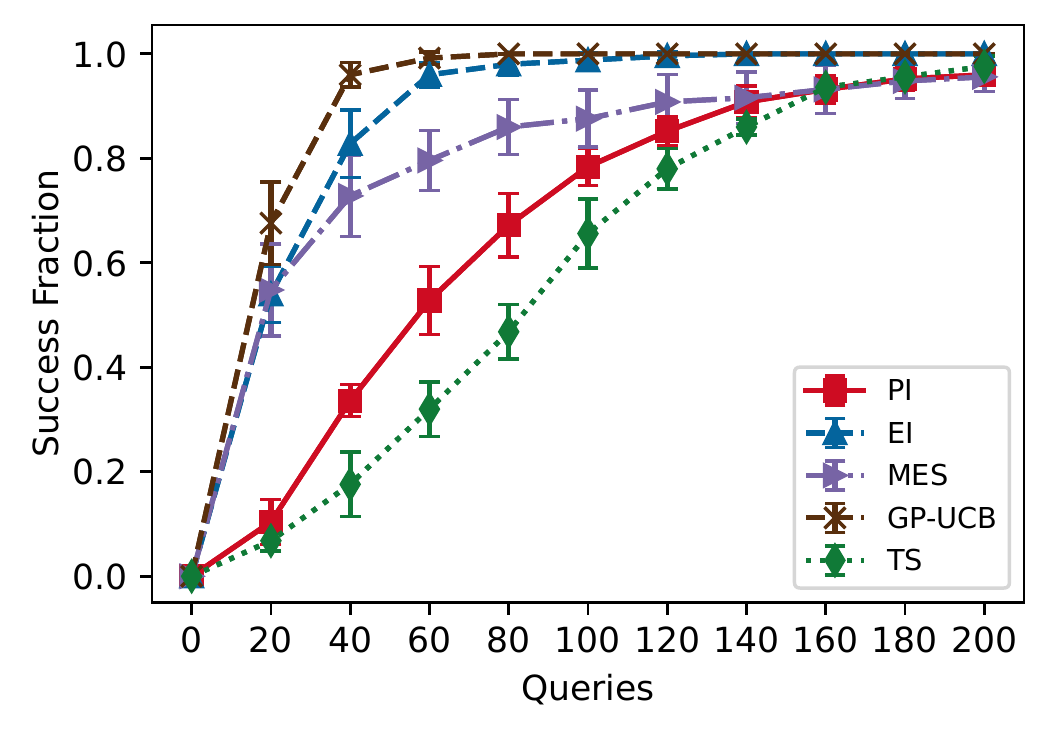}
        \includegraphics[width=0.3\columnwidth]{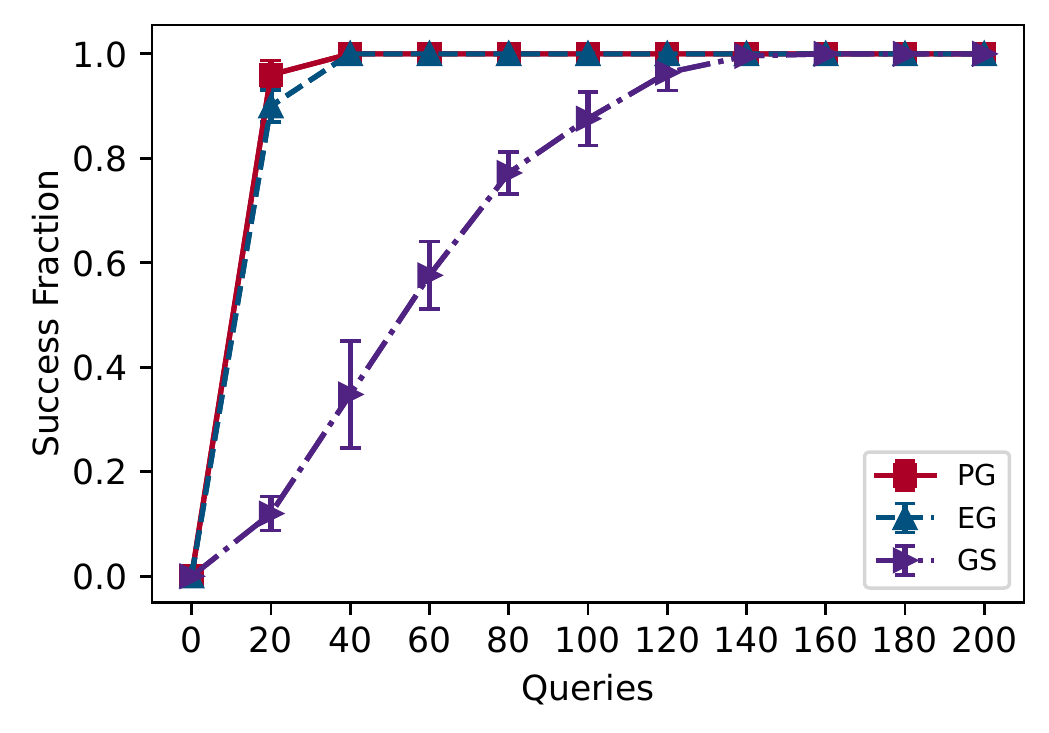}

        \caption{Ackley 6D with $\xi = \frac{1}{400}$.}
        \label{fig:Ackley_6_0.25}
    \end{subfigure}

    \vskip 0.1in
    \begin{subfigure}{\columnwidth}
        \centering
        \includegraphics[width=0.3\columnwidth]{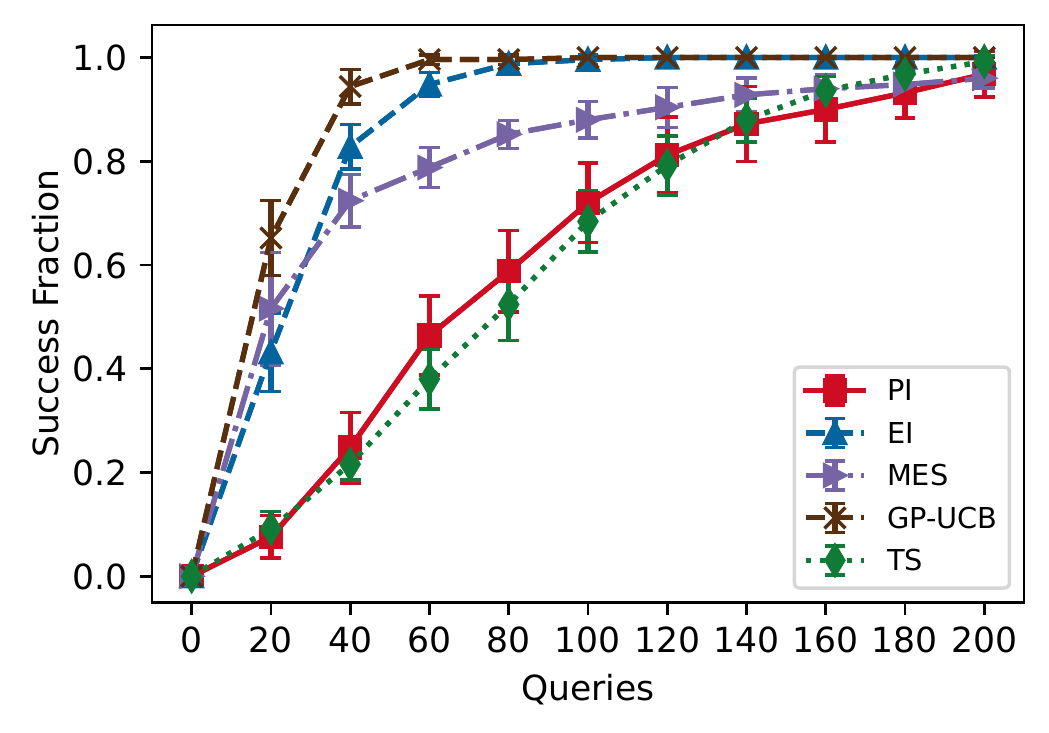}
        \includegraphics[width=0.3\columnwidth]{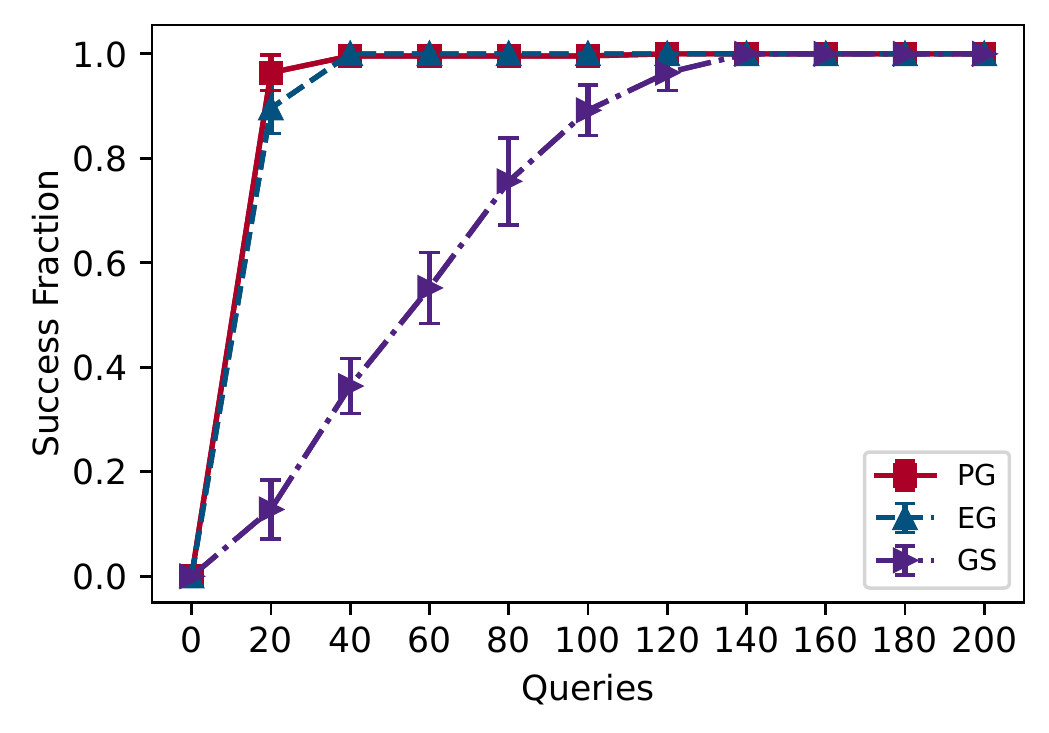}

        \caption{Ackley 6D with $\xi = \frac{1}{100}$.}
        \label{fig:Ackley_6_1}
    \end{subfigure}

    \vskip 0.1in
    \begin{subfigure}{\columnwidth}
        \centering
        \includegraphics[width=0.3\columnwidth]{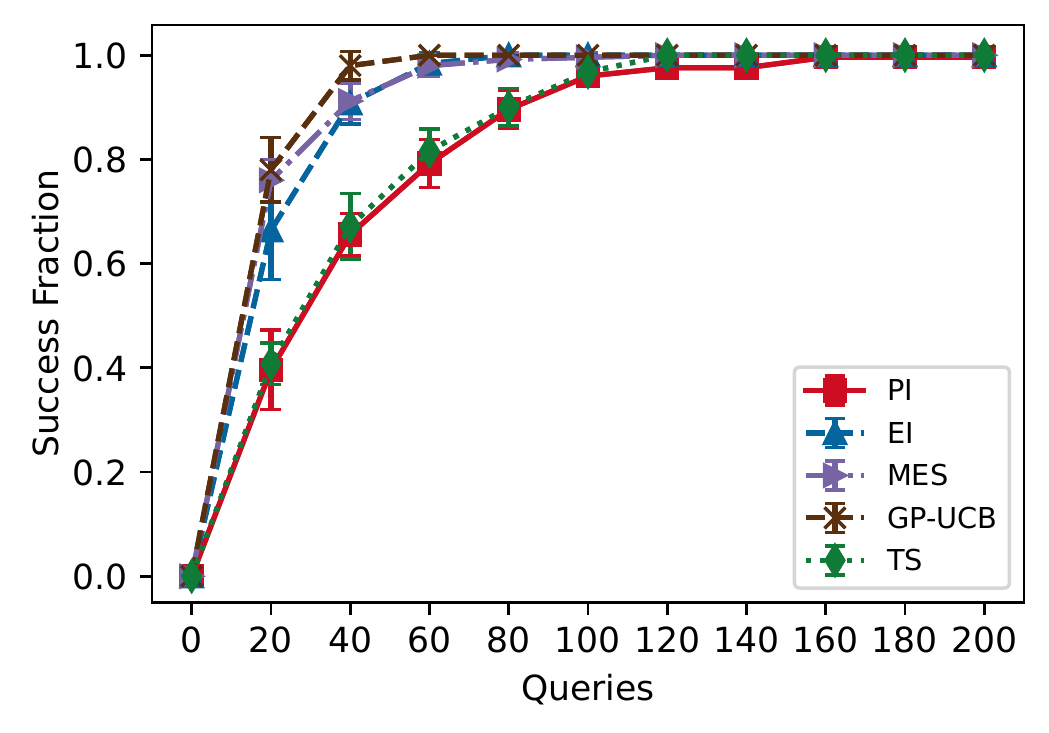}
        \includegraphics[width=0.3\columnwidth]{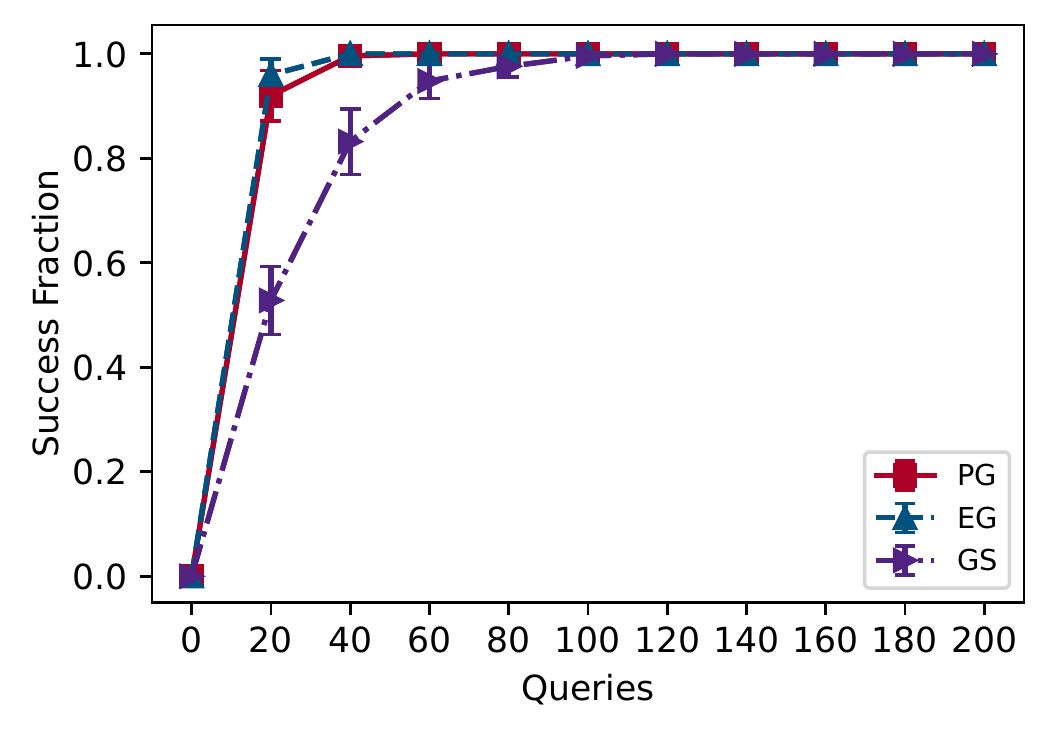}

        \caption{Ackley 6D with $\xi = \frac{1}{50}$.}
        \label{fig:Ackley_6_2}
    \end{subfigure}

    \caption{Ackley 6D function for different values of $\eta$ dictated by $\xi \in (0,1)$, the approximate proportion of points that are good. \label{fig:a6_for_diff_k}}
\end{figure}

\begin{figure}
    \centering
    \begin{subfigure}{\columnwidth}
        \centering
        \includegraphics[width=0.3\columnwidth]{figs/robot3d_eta_4p75_opt.pdf}
        \includegraphics[width=0.3\columnwidth]{figs/robot3d_eta_4p75_good.pdf}

        \caption{Robot Pushing 3D with $\eta=4.75$}
        \label{fig:Robot_Pushing_3D_0.25}
    \end{subfigure}

    \vskip 0.1in
    \begin{subfigure}{\columnwidth}
        \centering
        \includegraphics[width=0.3\columnwidth]{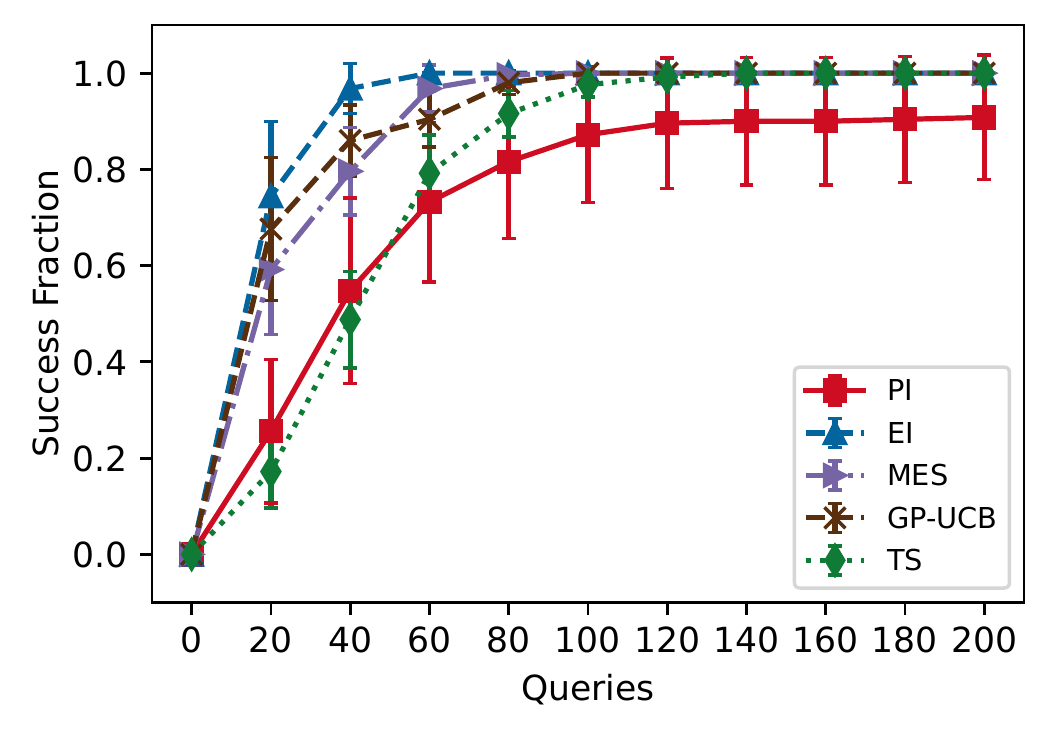}
        \includegraphics[width=0.3\columnwidth]{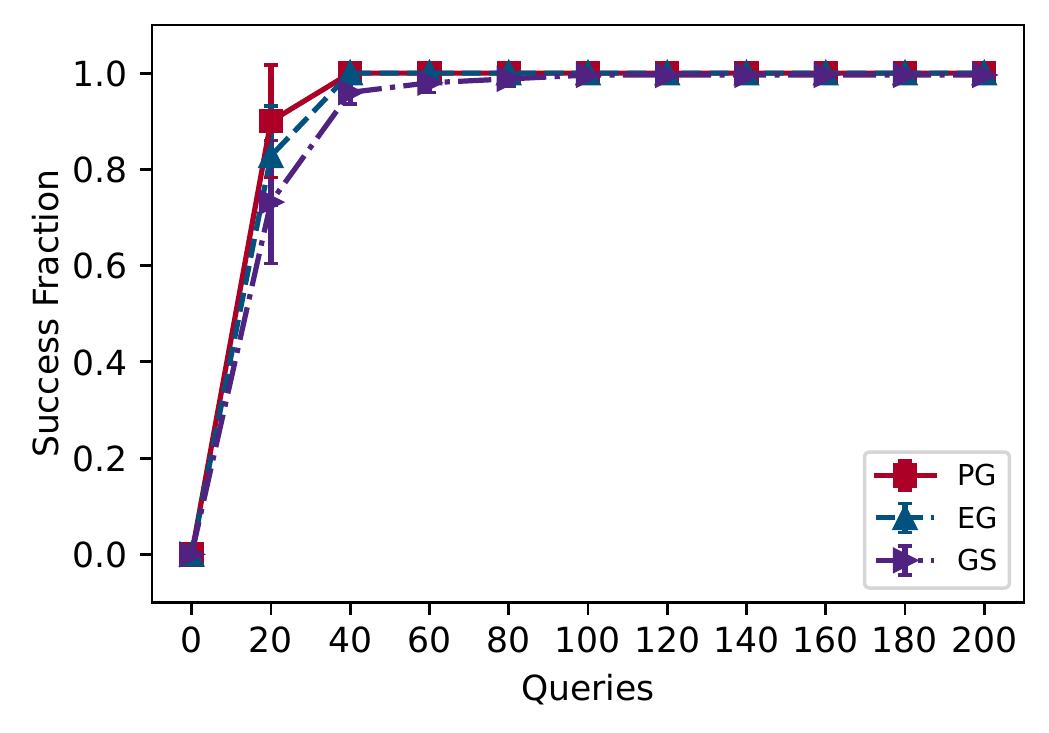}

        \caption{Robot Pushing 3D with $\eta=4.5$}
        \label{fig:Robot_Pushing_3D_1}
    \end{subfigure}

    \vskip 0.1in
    \begin{subfigure}{\columnwidth}
        \centering
        \includegraphics[width=0.3\columnwidth]{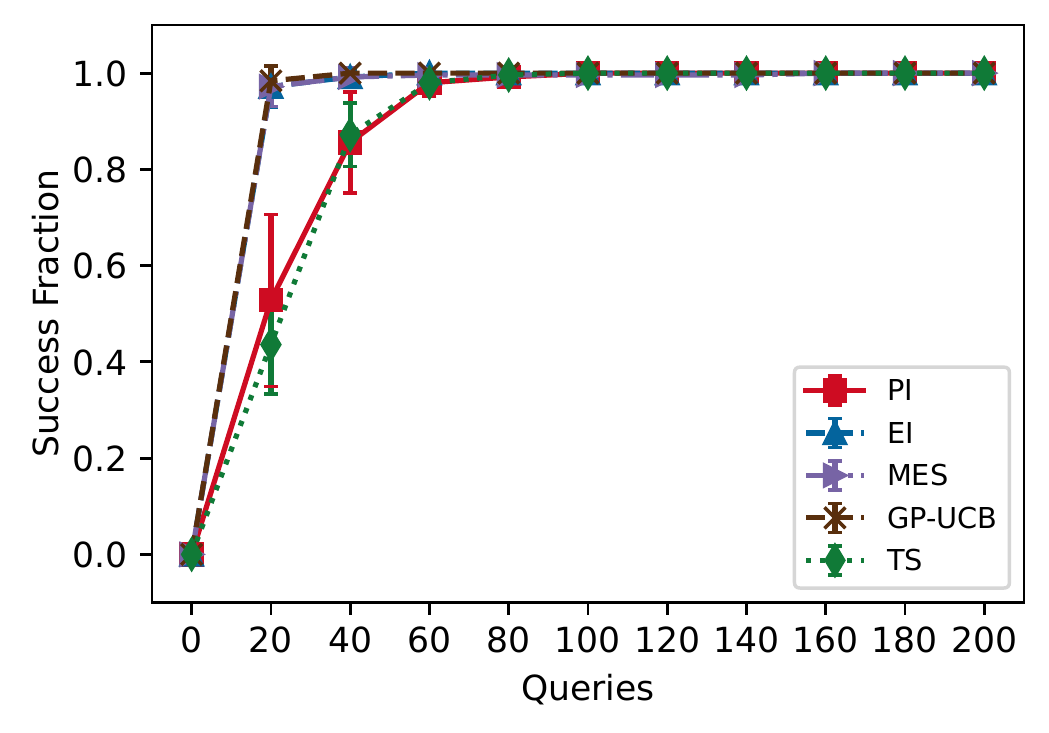}
        \includegraphics[width=0.3\columnwidth]{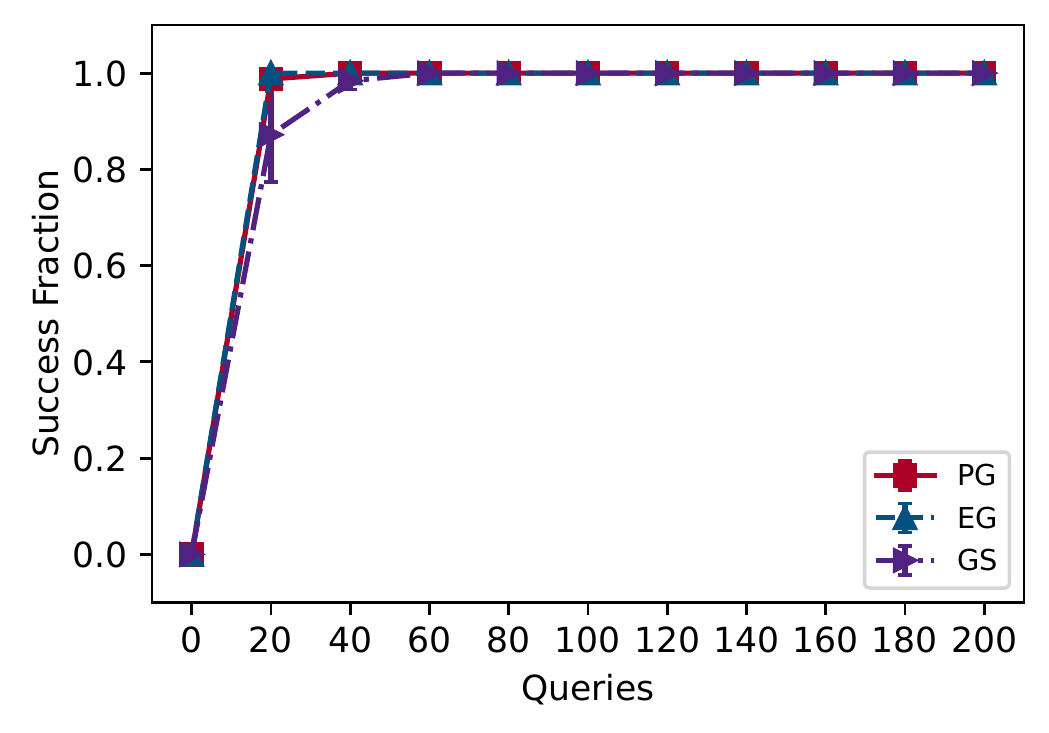}

        \caption{Robot Pushing 3D with $\eta=4.0$}
        \label{fig:Robot_Pushing_3D_2}
    \end{subfigure}

    \caption{Robot Pushing 3D function for different values of $\eta$ \label{fig:robot3_for_diff_k}}
\end{figure}

\begin{figure}
    \centering
    \begin{subfigure}{\columnwidth}
        \centering
        \includegraphics[width=0.3\columnwidth]{figs/robot4d_eta_4p75_opt.pdf}
        \includegraphics[width=0.3\columnwidth]{figs/robot4d_eta_4p75_good.pdf}

        \caption{Robot Pushing 4D with $\eta=4.75$}
        \label{fig:Robot_Pushing_3D_0.25a}
    \end{subfigure}

    \vskip 0.1in
    \begin{subfigure}{\columnwidth}
        \centering
        \includegraphics[width=0.3\columnwidth]{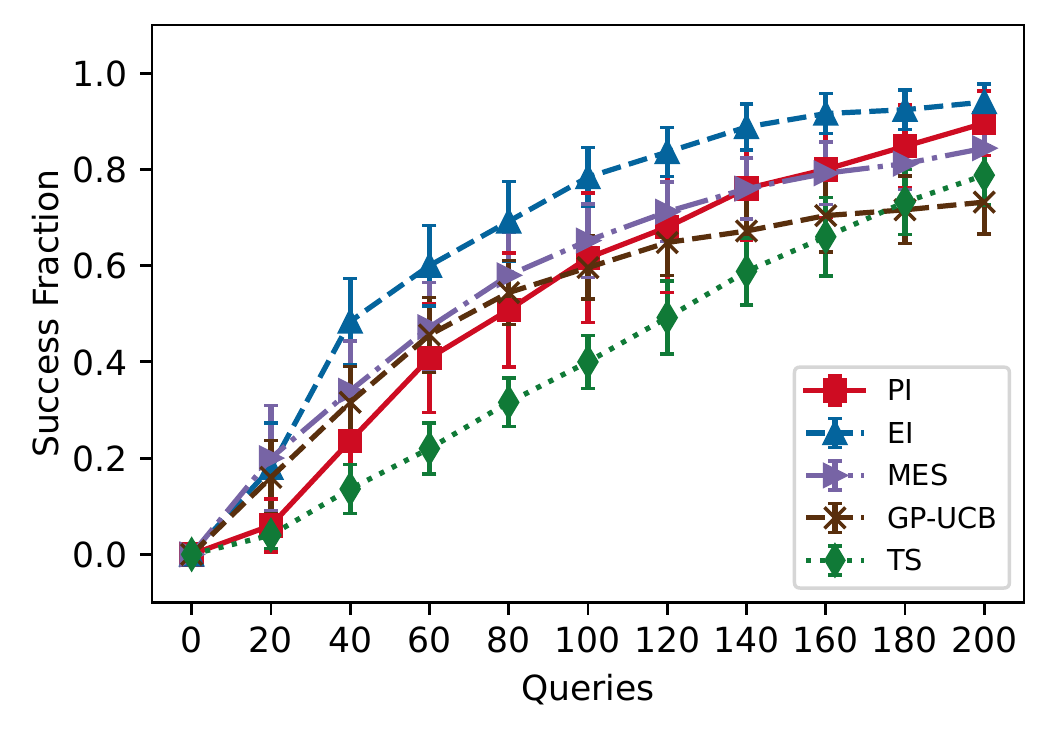}
        \includegraphics[width=0.3\columnwidth]{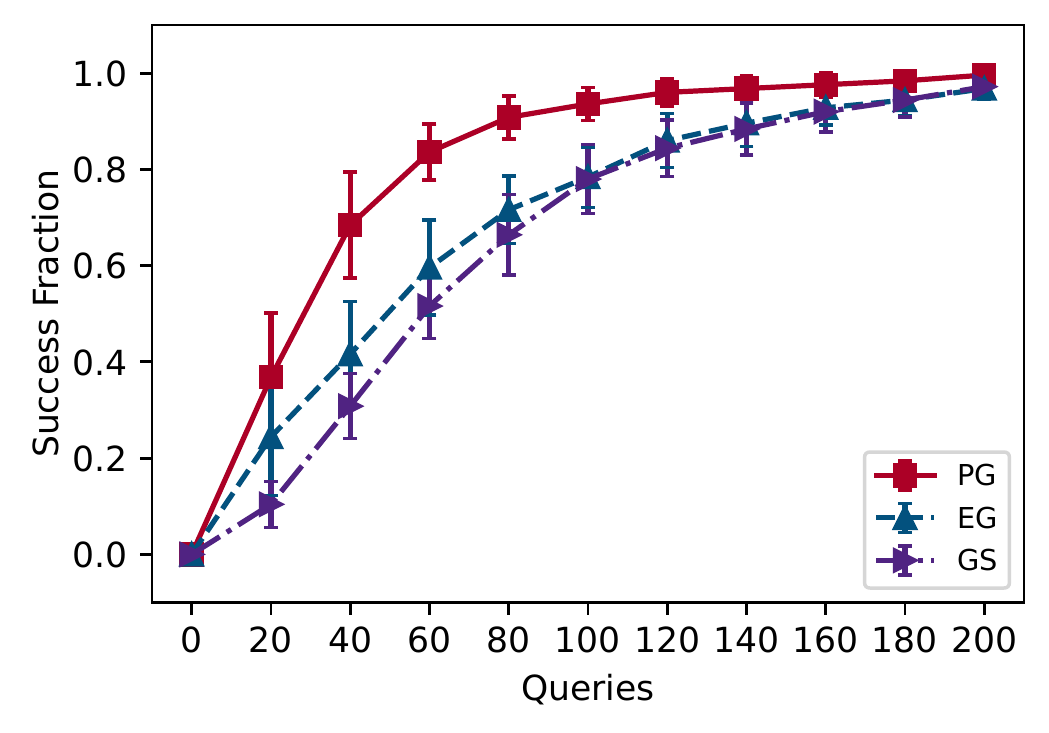}

        \caption{Robot Pushing 4D with $\eta=4.5$}
        \label{fig:Robot_Pushing_3D_1a}
    \end{subfigure}

    \vskip 0.1in
    \begin{subfigure}{\columnwidth}
        \centering
        \includegraphics[width=0.3\columnwidth]{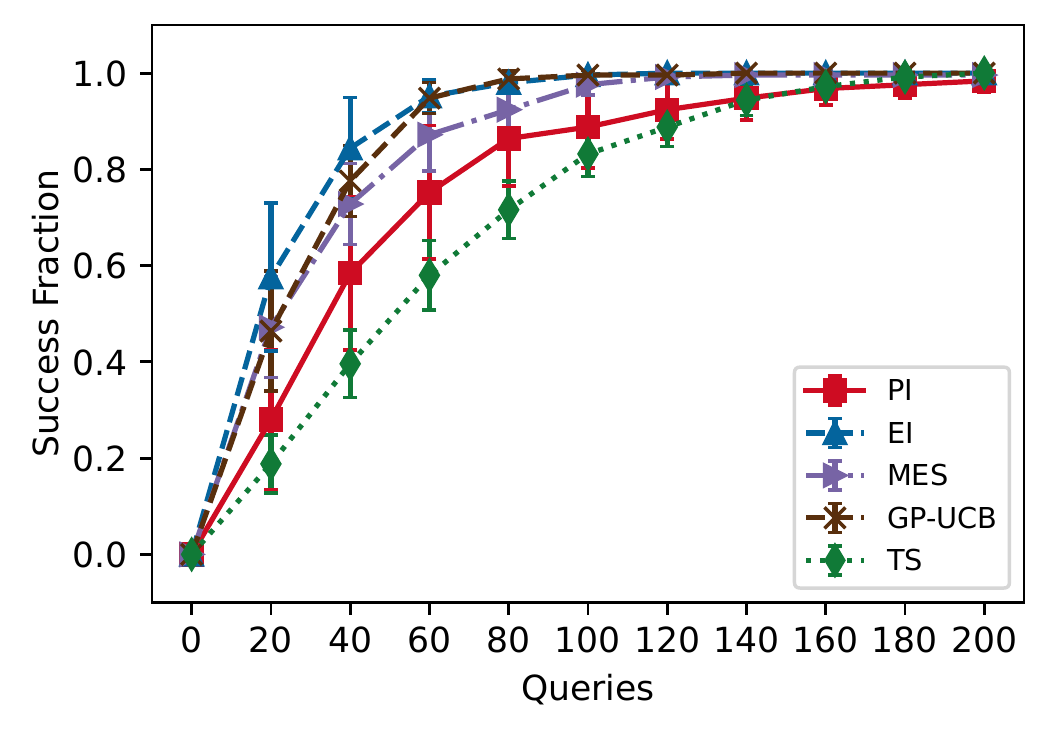}
        \includegraphics[width=0.3\columnwidth]{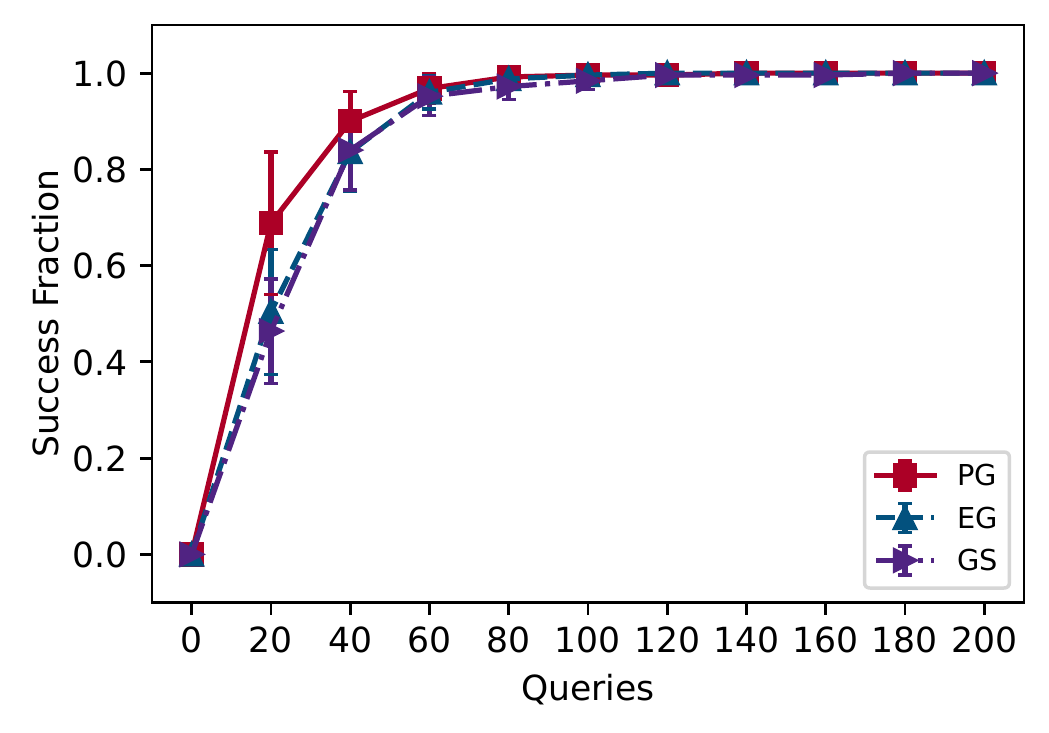}

        \caption{Robot Pushing 4D with $\eta=4.0$}
        \label{fig:Robot_Pushing_3D_2a}
    \end{subfigure}

    \caption{Robot Pushing 4D function for different values of $\eta$ \label{fig:robot4_for_diff_k}}
\end{figure}

\subsection{Cases When No Good Action Exists} \label{eq:no_good}

A potential concern of the good-action identification perspective is whether the algorithms can still be expected to behave in a reasonable manner when no good actions exist.  Here we provide evidence that, in fact, one can still maintain robustness, in the sense that even when $\eta > f(\xv^*)$, the algorithms introduced in Section \ref{sec:alg} can still find an action with function value close to $f(\xv^*)$.  To demonstrate this, we revert to the standard simple regret notion (since the ``fraction found'' notion used previously will always be zero here).

Figure \ref{fig:h3_for_diff_k} plots the simple regret for the 3D Hartmann function (with $f(\xv^*) = 3.863$). 
In sub-figure (a), we consider both $\eta$ slightly above the threshold, and significantly above.  Even in the latter case, PG and EG are able to attain simple regret tending to zero, indicating their robustness in the case that no good points exist.  While GS appears to be somewhat less robust, this could potentially be remedied by modifying how the algorithm behaves when all acquisition functions are zero, as discussed in Section \ref{sec:gs}.

An analogous plot for the robot pushing experiment is given in Figure \ref{fig:robot_for_diff_k}, with similar findings.  We note that the poor performance of PI here is due to the existence of a small number of runs in which the algorithm gets stuck in a highly suboptimal local minimum.  These runs significantly impact the average regret, but only have a minor impact on the cumulative fraction found in Figure \ref{fig:cumulative_plot_real} (due to occurring on few runs).

\begin{figure}
    \centering
    \begin{subfigure}{0.66\columnwidth}
        \centering
        \includegraphics[width=0.47\columnwidth]{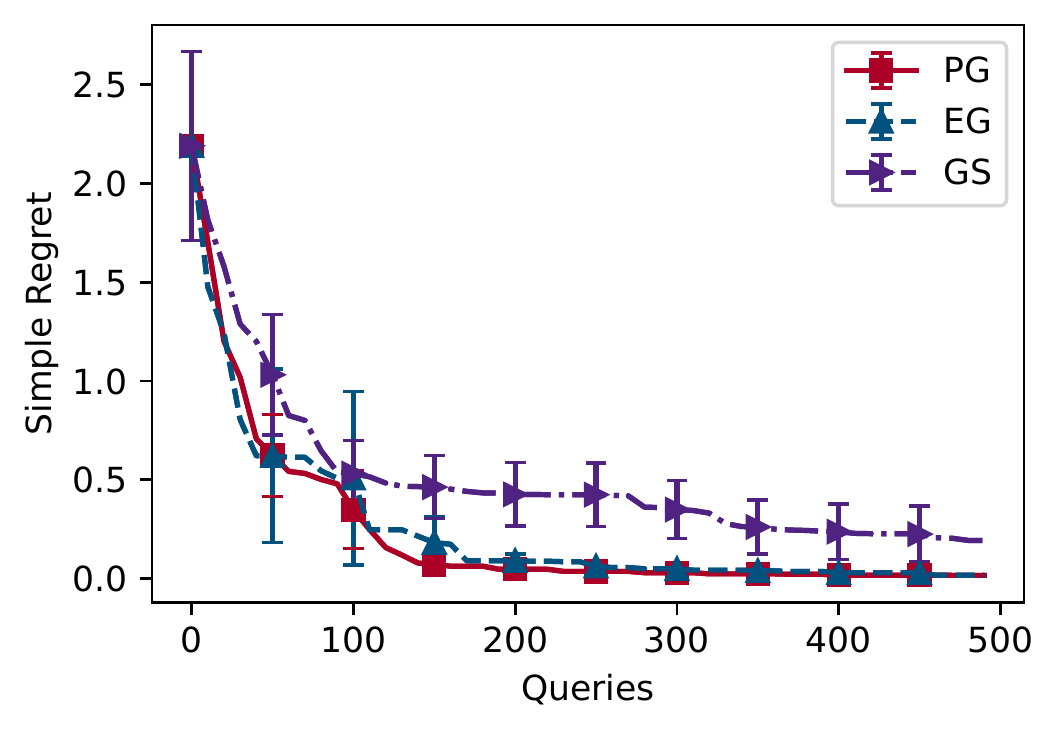}
        \includegraphics[width=0.47\columnwidth]{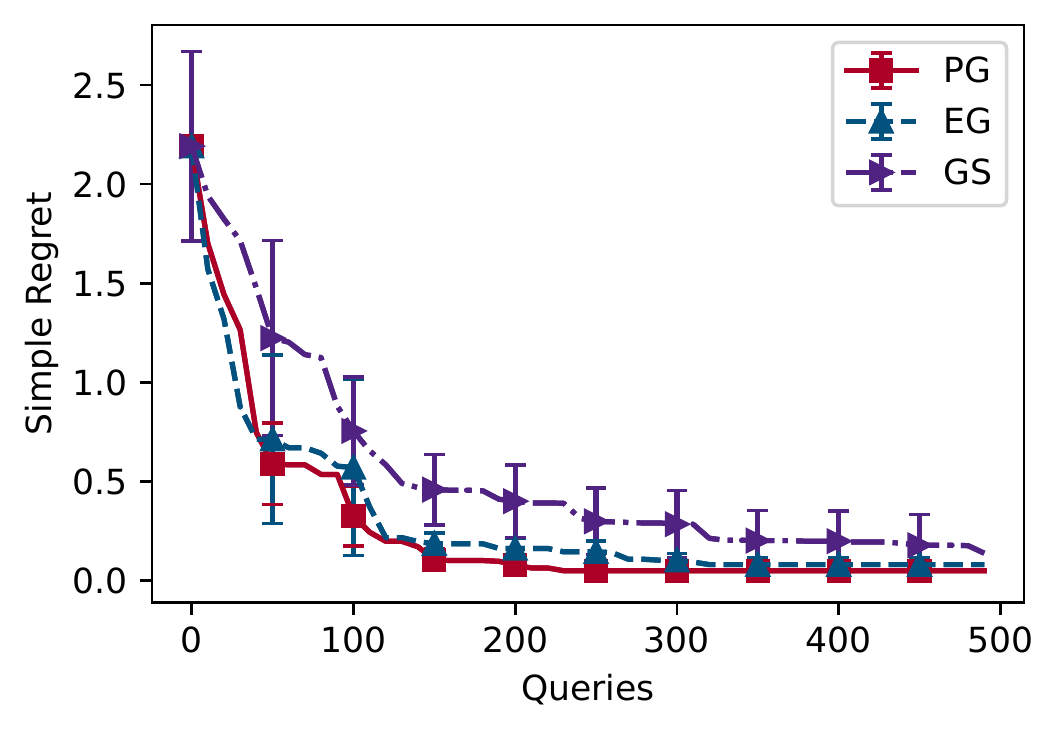}

        \caption{Our algorithms with $\eta = f(\xv^*) + 0.1$ (Left) and $\eta = f(\xv^*) + 0.5$ (Right).}
        \label{fig:Hartmann_3_Regret_GA}
    \end{subfigure}
    \begin{subfigure}{0.33\columnwidth}
        \centering
        \includegraphics[width=0.95\columnwidth]{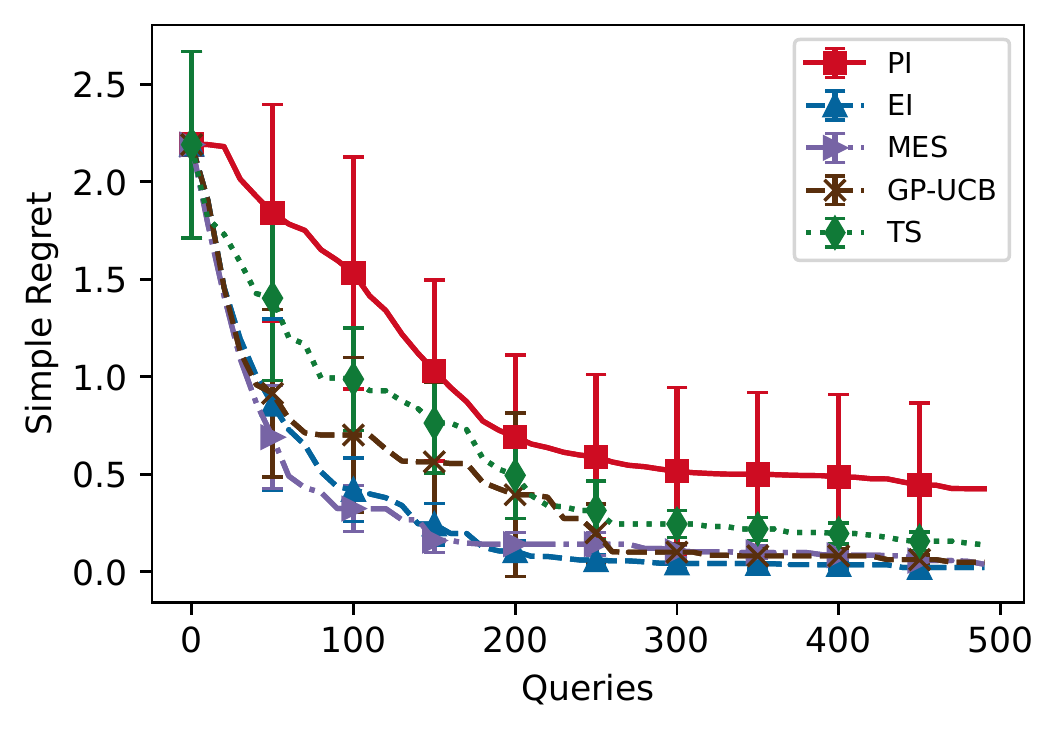}

        \caption{Standard optimization algorithms.}
        \label{fig:Hartmann_3_Regret_opt}
    \end{subfigure}

    \caption{Simple regret plots for the 3D Hartmann function when no good action exists. \label{fig:h3_for_diff_k}}
\end{figure}

\begin{figure}
    \centering
    \begin{subfigure}{0.66\columnwidth}
        \centering
        \includegraphics[width=0.47\columnwidth]{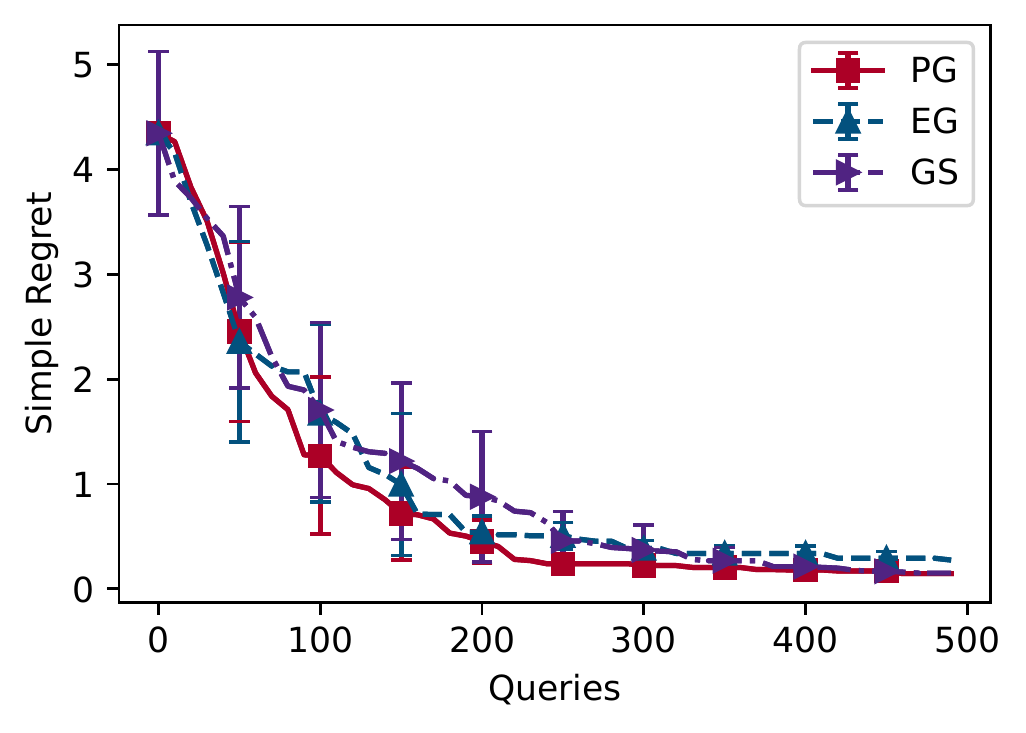}
        \includegraphics[width=0.47\columnwidth]{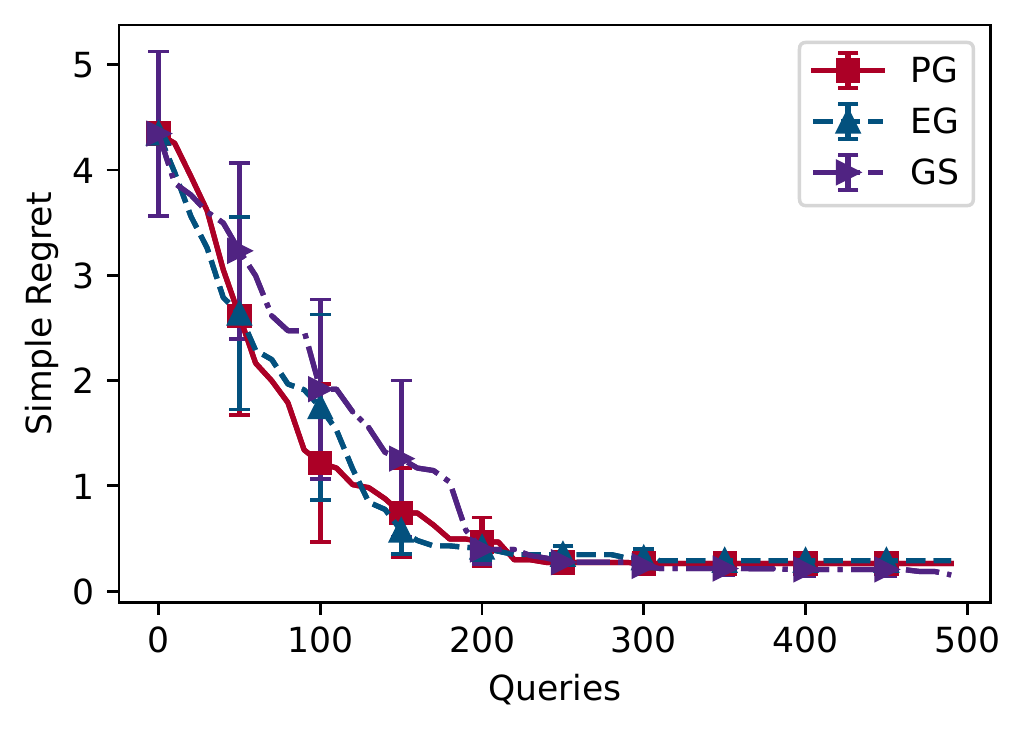}

        \caption{Our algorithms with $\eta = f(\xv^*) + 0.1$ (Left) and $\eta = f(\xv^*) + 0.5$ (Right).}
        \label{fig:Robot_Pushing_3D_Regret_GA}
    \end{subfigure}
    \begin{subfigure}{0.33\columnwidth}
        \centering
        \includegraphics[width=0.95\columnwidth]{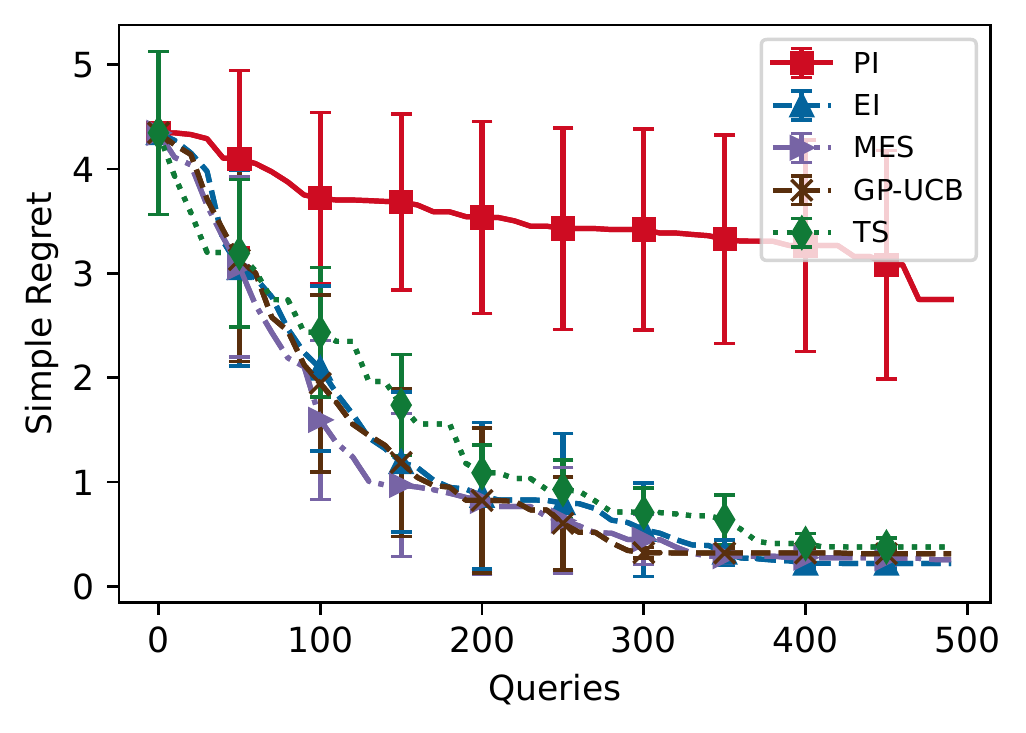}

        \caption{Standard optimization algorithms.}
        \label{fig:Robot_Pushing_3D_Regret_opt}
    \end{subfigure}

    \caption{Simple regret plots for the Robot Pushing 3D function when no good action exists. \label{fig:robot_for_diff_k}}
\end{figure}

%% file: GoodAction_ICML.bbl
\begin{thebibliography}{37}
\providecommand{\natexlab}[1]{#1}
\providecommand{\url}[1]{\texttt{#1}}
\expandafter\ifx\csname urlstyle\endcsname\relax
  \providecommand{\doi}[1]{doi: #1}\else
  \providecommand{\doi}{doi: \begingroup \urlstyle{rm}\Url}\fi

\bibitem[Abbasi-Yadkori(2013)]{AbbasiThesis}
Abbasi-Yadkori, Y.
\newblock \emph{Online learning for linearly parametrized control problems}.
\newblock PhD thesis, University of Alberta, 2013.

\bibitem[Bingham(2021)]{SFU_Funcs}
Bingham, D.
\newblock Virtual library of simulation experiments: Test functions and
  datasets.
\newblock \url{https://www.sfu.ca/~ssurjano/index.html}, 2021.

\bibitem[Bogunovic et~al.(2016)Bogunovic, Scarlett, Krause, and Cevher]{Bog16a}
Bogunovic, I., Scarlett, J., Krause, A., and Cevher, V.
\newblock Truncated variance reduction: A unified approach to {B}ayesian
  optimization and level-set estimation.
\newblock In \emph{Conf. Neur. Inf. Proc. Sys. (NeurIPS)}, 2016.

\bibitem[Bogunovic et~al.(2020)Bogunovic, Krause, and Scarlett]{Bog20}
Bogunovic, I., Krause, A., and Scarlett, J.
\newblock Corruption-tolerant {G}aussian process bandit optimization.
\newblock In \emph{Int. Conf. Art. Intel. Stats. (AISTATS)}, 2020.

\bibitem[Bryan et~al.(2006)Bryan, Nichol, Genovese, Schneider, Miller, and
  Wasserman]{Bry05}
Bryan, B., Nichol, R.~C., Genovese, C.~R., Schneider, J., Miller, C.~J., and
  Wasserman, L.
\newblock Active learning for identifying function threshold boundaries.
\newblock In \emph{Conf. Neur. Inf. Proc. Sys. (NeurIPS)}, 2006.

\bibitem[Bull(2011)]{Bul11}
Bull, A.~D.
\newblock Convergence rates of efficient global optimization algorithms.
\newblock \emph{J. Mach. Learn. Res.}, 12\penalty0 (Oct.):\penalty0 2879--2904,
  2011.

\bibitem[Cai \& Scarlett(2021)Cai and Scarlett]{Cai20}
Cai, X. and Scarlett, J.
\newblock On lower bounds for standard and robust {G}aussian process bandit
  optimization.
\newblock In \emph{Int. Conf. Mach. Learn. (ICML)}, 2021.

\bibitem[Chen \& Guestrin(2016)Chen and Guestrin]{chen16}
Chen, T. and Guestrin, C.
\newblock Xgboost: A scalable tree boosting system.
\newblock In \emph{{ACM} SIGKDD Int. Conf. Knowl. Disc. Data Mining}, pp.\
  785--794, 2016.

\bibitem[Chowdhury \& Gopalan(2017)Chowdhury and Gopalan]{Cho17}
Chowdhury, S.~R. and Gopalan, A.
\newblock On kernelized multi-armed bandits.
\newblock In \emph{Int. Conf. Mach. Learn. (ICML)}, 2017.

\bibitem[Contal et~al.(2013)Contal, Buffoni, Robicquet, and Vayatis]{Con13}
Contal, E., Buffoni, D., Robicquet, A., and Vayatis, N.
\newblock \emph{Machine Learning and Knowledge Discovery in Databases}, chapter
  Parallel {G}aussian Process Optimization with Upper Confidence Bound and Pure
  Exploration, pp.\  225--240.
\newblock Springer Berlin Heidelberg, 2013.

\bibitem[Gotovos et~al.(2013)Gotovos, Casati, Hitz, and Krause]{Got13}
Gotovos, A., Casati, N., Hitz, G., and Krause, A.
\newblock Active learning for level set estimation.
\newblock In \emph{Int. Joint. Conf. Art. Intel. (IJCAI)}, 2013.

\bibitem[Gr{\"u}new{\"a}lder et~al.(2010)Gr{\"u}new{\"a}lder, Audibert, Opper,
  and Shawe-Taylor]{Gru10}
Gr{\"u}new{\"a}lder, S., Audibert, J.-Y., Opper, M., and Shawe-Taylor, J.
\newblock Regret bounds for {G}aussian process bandit problems.
\newblock In \emph{Int. Conf. Art. Intel. Stats. (AISTATS)}, pp.\  273--280,
  2010.

\bibitem[Hennig \& Schuler(2012)Hennig and Schuler]{Hen12}
Hennig, P. and Schuler, C.~J.
\newblock Entropy search for information-efficient global optimization.
\newblock \emph{J. Mach. Learn. Research}, 13\penalty0 (1):\penalty0
  1809--1837, 2012.

\bibitem[Hern{\'a}ndez-Lobato et~al.(2014)Hern{\'a}ndez-Lobato, Hoffman, and
  Ghahramani]{Her14}
Hern{\'a}ndez-Lobato, J.~M., Hoffman, M.~W., and Ghahramani, Z.
\newblock Predictive entropy search for efficient global optimization of
  black-box functions.
\newblock In \emph{Conf. Neur. Inf. Proc. Sys. (NeurIPS)}, 2014.

\bibitem[Janz et~al.(2020)Janz, Burt, and Gonz\'alez]{Jan20}
Janz, D., Burt, D.~R., and Gonz\'alez, J.
\newblock Bandit optimisation of functions in the {M}at\'ern kernel {RKHS}.
\newblock In \emph{Int. Conf. Art. Intel. Stats. (AISTATS)}, 2020.

\bibitem[Kano et~al.(2019)Kano, Honda, Sakamaki, Matsuura, Nakamura, and
  Sugiyama]{Kan19}
Kano, H., Honda, J., Sakamaki, K., Matsuura, K., Nakamura, A., and Sugiyama, M.
\newblock Good arm identification via bandit feedback.
\newblock \emph{Machine Learning}, 108\penalty0 (5):\penalty0 721--745, 2019.

\bibitem[Katz-Samuels \& Jamieson(2020)Katz-Samuels and Jamieson]{Kat20}
Katz-Samuels, J. and Jamieson, K.
\newblock The true sample complexity of identifying good arms.
\newblock In \emph{Int. Conf. Art. Intel. Stats. (AISTATS)}, 2020.

\bibitem[Kushner(1964)]{Kus64}
Kushner, H.~J.
\newblock A new method of locating the maximum point of an arbitrary multipeak
  curve in the presence of noise.
\newblock \emph{J. Fluids Eng.}, 86\penalty0 (1):\penalty0 97--106, 1964.

\bibitem[Lattimore \& Szepesv\'ari(2020)Lattimore and Szepesv\'ari]{Csa18}
Lattimore, T. and Szepesv\'ari, C.
\newblock \emph{Bandit Algorithms}.
\newblock Cambridge University Press, 2020.

\bibitem[Li et~al.(2017)Li, Rana, Gupta, Nguyen, and Venkatesh]{Li17}
Li, C., Rana, S., Gupta, S., Nguyen, V., and Venkatesh, S.
\newblock Bayesian optimization with monotonicity information.
\newblock In \emph{{NeurIPS} Workshop on {B}ayesian Optimization}, 2017.

\bibitem[Merlis \& Mannor(2021)Merlis and Mannor]{Mer20}
Merlis, N. and Mannor, S.
\newblock Lenient regret for multi-armed bandits.
\newblock In \emph{AAAI Conf. Art. Intel.}, 2021.

\bibitem[Mockus et~al.(1978)Mockus, Tiesis, and Zilinskas]{Moc78}
Mockus, J., Tiesis, V., and Zilinskas, A.
\newblock The application of bayesian methods for seeking the extremum.
\newblock \emph{Towards Global Optimization}, 2\penalty0 (117-129):\penalty0 2,
  1978.

\bibitem[Nguyen \& Osborne(2020)Nguyen and Osborne]{Ngu20a}
Nguyen, V. and Osborne, M.~A.
\newblock Knowing the what but not the where in {B}ayesian optimization.
\newblock In \emph{Int. Conf. Mach. Learn. (ICML)}. PMLR, 2020.

\bibitem[Rasmussen(2006)]{Ras06}
Rasmussen, C.~E.
\newblock Gaussian processes for machine learning.
\newblock MIT Press, 2006.

\bibitem[Russo \& Van~Roy(2018)Russo and Van~Roy]{Rus18}
Russo, D. and Van~Roy, B.
\newblock Satisficing in time-sensitive bandit learning.
\newblock https://arxiv.org/abs/1803.02855, 2018.

\bibitem[Russo et~al.(2018)Russo, Van~Roy, Kazerouni, Osband, and Wen]{Rus18a}
Russo, D.~J., Van~Roy, B., Kazerouni, A., Osband, I., and Wen, Z.
\newblock A tutorial on thompson sampling.
\newblock \emph{Found. Trends Mach. Learn.}, 11\penalty0 (1):\penalty0 1–96,
  2018.

\bibitem[Scarlett(2018)]{Sca18a}
Scarlett, J.
\newblock Tight regret bounds for {B}ayesian optimization in one dimension.
\newblock In \emph{Int. Conf. Mach. Learn. (ICML)}, 2018.

\bibitem[Scarlett et~al.(2017)Scarlett, Bogunovic, and Cevher]{Sca17a}
Scarlett, J., Bogunovic, I., and Cevher, V.
\newblock Lower bounds on regret for noisy {G}aussian process bandit
  optimization.
\newblock In \emph{Conf. Learn. Theory (COLT)}. 2017.

\bibitem[Shahriari et~al.(2016)Shahriari, Swersky, Wang, Adams, and
  de~Freitas]{Sha16}
Shahriari, B., Swersky, K., Wang, Z., Adams, R.~P., and de~Freitas, N.
\newblock Taking the human out of the loop: A review of {B}ayesian
  optimization.
\newblock \emph{Proc. IEEE}, 104\penalty0 (1):\penalty0 148--175, 2016.

\bibitem[Shekhar \& Javidi(2018)Shekhar and Javidi]{She17}
Shekhar, S. and Javidi, T.
\newblock Gaussian process bandits with adaptive discretization.
\newblock \emph{Elec. J. Stats.}, 12\penalty0 (2):\penalty0 3829--3874, 2018.

\bibitem[Shekhar \& Javidi(2019)Shekhar and Javidi]{She19}
Shekhar, S. and Javidi, T.
\newblock Multiscale {G}aussian process level set estimation.
\newblock In \emph{Proc. Mach. Learn. Research}, volume~89, pp.\  3283--3291,
  April 2019.

\bibitem[Shekhar \& Javidi(2020)Shekhar and Javidi]{She20}
Shekhar, S. and Javidi, T.
\newblock Multi-scale zero-order optimization of smooth functions in an {RKHS}.
\newblock https://arxiv.org/abs/2005.04832, 2020.

\bibitem[Srinivas et~al.(2010)Srinivas, Krause, Kakade, and Seeger]{Sri09}
Srinivas, N., Krause, A., Kakade, S.~M., and Seeger, M.
\newblock Gaussian process optimization in the bandit setting: No regret and
  experimental design.
\newblock In \emph{Int. Conf. Mach. Learn. (ICML)}, 2010.

\bibitem[Sui et~al.(2015)Sui, Gotovos, Burdick, and Krause]{Sui15}
Sui, Y., Gotovos, A., Burdick, J.~W., and Krause, A.
\newblock Safe exploration for optimization with {G}aussian processes.
\newblock In \emph{Int. Conf. Mach. Learn. (ICML)}, 2015.

\bibitem[Vakili et~al.(2021)Vakili, Khezeli, and Picheny]{Vak20a}
Vakili, S., Khezeli, K., and Picheny, V.
\newblock On information gain and regret bounds in {G}aussian process bandits.
\newblock In \emph{Int. Conf. Art. Intel. Stats. (AISTATS)}, 2021.

\bibitem[Valko et~al.(2013)Valko, Korda, Munos, Flaounas, and
  Cristianini]{Val13}
Valko, M., Korda, N., Munos, R., Flaounas, I., and Cristianini, N.
\newblock Finite-time analysis of kernelised contextual bandits.
\newblock In \emph{Conf. Uncertainty in AI (UAI)}, 2013.

\bibitem[Wang \& Jegelka(2017)Wang and Jegelka]{Wan17}
Wang, Z. and Jegelka, S.
\newblock Max-value entropy search for efficient {B}ayesian optimization.
\newblock In \emph{Int. Conf. Mach. Learn. (ICML)}, pp.\  3627--3635, 2017.

\end{thebibliography}
